%% file: arxiv_main.tex
\documentclass[11pt]{article}

\usepackage[letterpaper, margin=1in]{geometry}
\usepackage{microtype}
\usepackage{graphicx}
\usepackage{subfigure}
\usepackage{booktabs} 
\usepackage{hyperref}
\PassOptionsToPackage{sort, compress}{natbib}
\usepackage{natbib}

\hypersetup{
    unicode=false,          
    pdftoolbar=true,        
    pdfmenubar=true,        
    pdffitwindow=false,     
    pdfstartview={FitH},    
    pdfnewwindow=true,      
    colorlinks=true,       
    linkcolor=blue,          
    citecolor=black,        
    filecolor=magenta,      
    urlcolor=blue,           
    breaklinks=true
}

\usepackage{amsmath}
\usepackage{amssymb}
\usepackage{mathtools}
\usepackage{amsthm}
\usepackage{mathrsfs}
\usepackage{thmtools, thm-restate}

\usepackage[capitalize,noabbrev]{cleveref}

\usepackage{Shorthands2}

\usepackage{ifthen}
\newif\ifshortver
\newcommand{\ifshort}[2]{\ifshortver#1\else#2\fi}

\theoremstyle{plain}
\newtheorem{theorem}{Theorem}[section]
\newtheorem{proposition}[theorem]{Proposition}
\newtheorem{lemma}[theorem]{Lemma}

\newtheorem{definition}[theorem]{Definition}
\newtheorem{assumption}[theorem]{Assumption}
\theoremstyle{definition}
\newtheorem{remark}[theorem]{Remark}

\usepackage[textsize=tiny]{todonotes}

\title{Guarantees for Nonlinear Representation Learning: \\Non-identical Covariates, Dependent Data, Fewer Samples}

\author{Thomas T.\ Zhang\footnote{Corresponding author, email: \texttt{ttz2@seas.upenn.edu}}\;, Bruce D.\ Lee, Ingvar Ziemann, George J.\ Pappas, Nikolai Matni}
\date{Department of Electrical and Systems Engineering, University of Pennsylvania}

\begin{document}

\maketitle

\begin{abstract}
A driving force behind the diverse applicability of modern machine learning is the ability to extract meaningful features across many sources. However, many practical domains involve data that are non-identically distributed across sources, and possibly statistically dependent within its source, violating vital assumptions in existing theoretical studies of representation learning. Toward addressing these issues, we establish statistical guarantees for learning general \textit{nonlinear} representations from multiple  data sources that admit different input distributions and possibly sequentially dependent data.  
Specifically, we study the sample-complexity of learning $T+1$ functions $f_\star^{(t)} \circ g_\star$ from a function class $\mathcal F \times \mathcal G$, where $f_\star^{(t)}$ are task specific linear functions and $g_\star$ is a shared non-linear representation. An approximate representation $\hat g$ is estimated using $N$ samples from each of $T$ source tasks, and a fine-tuning function $\hat f^{(0)}$ is fit using $N$ samples from a target task passed through $\hat g$. 
Our results show that when $N \gtrsim C_{\mathrm{dep}} (\mathrm{dim}(\mathcal F) + \mathrm{C}(\mathcal G)/T)$, the excess risk of the estimate $\hat f^{(0)} \circ \hat g$ on the target task decays as $\nu_{\mathrm{div}} \big(\frac{\mathrm{dim}(\mathcal F)}{N} + \frac{\mathrm{C}(\mathcal G)}{N T} \big)$, where $C_{\mathrm{dep}}$ denotes the effect of data dependency, $\nu_{\mathrm{div}}$ denotes an (estimatable) measure of \textit{task-diversity} between the source and target tasks, and $\mathrm{C}(\mathcal G)$ denotes a complexity measure of the representation class $\calG$. In particular, our analysis reveals: 1.\ as the number of tasks $T$ increases, both the sample requirement and risk bound converge to that of $r$-dimensional regression as if $g_\star$ had been given, 2.\ the effect of dependency only enters the sample requirement, leaving the risk bound matching the iid setting, and 3.\ the proposed task diversity measure $\nu_{\mathrm{div}}$ addresses pathologies like ill-conditioning and rank-degeneracy while avoiding direct uniformity assumptions.
\end{abstract}

\input{icml/icml_tex/intro_background}

\input{icml/icml_tex/theory_results}


\input{icml/icml_tex/numerical}

\input{icml/icml_tex/discussion_conclusion}

\section*{Acknowledgements}
Ingvar Ziemann is supported by a Swedish Research Council international postdoc grant. George J. Pappas is supported in part by NSF Award SLES-2331880. Nikolai Matni is supported in part by NSF Award SLES-2331880, NSF CAREER award ECCS-2045834, NSF EECS-2231349, and AFOSR Award FA9550-24-1-0102.


\bibliography{refs, bib2}
\bibliographystyle{abbrvnat}

\clearpage
\appendix
\onecolumn

\input{icml/icml_tex/appdx_theory_results}

\input{icml/icml_tex/mixingapp}

\input{icml/icml_tex/appdx_numerical}
\end{document}

%% file: icml/icml_tex/intro_background.tex
    







\ifshort{
\vspace{-0.8cm}
}

\section{Introduction}\label{sec: intro}

Transfer learning, in which a model is pre-trained on a large  dataset, and then finetuned for a specific application, has shown great success in various fields of machine learning including computer vision \citep{dosovitskiy2021an} and natural language processing \citep{devlin2019bert}. The principle enabling the success of these approaches is the use of a large dataset to extract compressed features which are broadly useful for downstream tasks. The extraction of such generally useful features from data is referred to as \emph{representation learning} \citep{bengio2013representation}. Despite its critical role in the success of deep learning, statistical guarantees remain somewhat limited.

Only recently have studies formalized multi-task representation learning in a way that illustrates how generalization improves when data is aggregated across many tasks \citep{du2020few, tripuraneni2020theory}. These regression settings consider learning $T+1$ functions $f_\star^{(t)} \circ g_\star$ in a function class $\calF \times \calG$ from covariate-observation pairs $\scurly{(\xit, \yit)}$, where $f_\star^{(t)}$ are task-specific functions, and $g_\star$ is a shared representation. The tasks for $t=1,\dots, T$ are denoted source (training) tasks, while $t=0$ is the target (test) task. A basic model of transfer learning can be expressed as a two-step procedure in which an estimate $\hat g$ for the representation is determined by solving  a least squares problem using $N$ data samples from each of the source tasks with measurements corrupted by zero-mean noise. This representation is then used to determine an estimate $\hat f^{(0)}$ by solving a least squares problem using $N'$ samples from the target task, also with measurements corrupted by zero-mean noise. \citet{du2020few, tripuraneni2020theory} show generalization bounds on the learned predictor in which the excess risk scales as $\tilde \calO \paren{\frac{\mathrm{C}(\calG)}{N T} + \frac{\mathrm{C}(\calF)}{N'}}$, where $\mathrm{C}$ quantifies the complexity of a function class. These rates capture the desirable behavior where the error from fitting the shared representation decays with the \textit{total} amount of data aggregated across the $T$ source tasks. 

While a rather complete picture can be stitched for linear settings, for such rates to hold in settings where the representation class $\calG$ is nonlinear, prior work crucially relies upon the assumption that covariates are independent and identically distributed (iid) across all tasks, such that the only source of variation comes from the task-specific $f_\star^{(t)}$. Such assumptions are fundamentally incompatible with many potential use cases of multi-task representation learning, such as in domain generalization and sequential decision-making. A key goal of this work is to remedy this issue and achieve multi-task rates in the absence of assumptions requiring identical covariate distributions across tasks, and independent data within tasks.


\ifshort{
\vspace{-0.3cm}
}

\subsection{Related Work}
\paragraph{Multi-task linear regression: } 

Beginning with \citet{du2020few}, a fairly complete picture has emerged in the setting of multi-task linear regression in which distinct tasks share a low dimensional representation, i.e. $f_\star^{(t)}(z) = F z$ and $g(x) = \Phi x$ for some matrices $F \in \mathbb{R}^{\dy \times r}$ and $\Phi \in \mathbb{R}^{r \times \dx} $ with $r \leq \dx$\footnote{\citet{du2020few} consider a scalar setting $\dy = 1$. Their analysis is extended to vector-valued settings in \citet{zhang2023multi}.}. In this setting, the authors demonstrate that the excess risk achieved by the empirical risk minimizer (ERM) achieves rates $\tilde \calO\paren{\frac{\dx r}{N T} + \frac{\dy r}{N'}}$. An active learning setting is considered by \citet{chen2022active, wang2023improved}, in which the assumption of uniform sampling from each task is replaced with an adaptive sampling algorithm. \citet{chua2021fine} considers a setting in which the representation is fine-tuned for each task, thereby allowing the assumption of a shared $\Phi$ to be relaxed.  Crucially, all of the aforementioned bounds hold only if the minimum amount of data (``burn-in time'') per task exceeds a quantity proportional to $\dx$. This is counterintuitive, as a goal of aggregating data across tasks is to remove the necessity for many samples per task. Furthermore, solving the ERM is nominally a non-convex bilinear problem. Efficient algorithms to bypass ERM have been proposed to explicitly address these issues  \citep{tripuraneni2021provable, collins2021exploiting, thekumparampil2021sample}, while attaining same rates order-wise. The resulting analysis alleviates the dependence of the burn-in time on $\dx$, but require iid covariates across all tasks, such that their estimators are consistent without requiring standardizing data per-task which would otherwise reintroduce a $\approx\dx$ burn-in per task. This is partially resolved by an algorithm proposed in \citet{zhang2023meta}, which handles tasks with non-identical covariate distributions; however the burn-in remains proportional to $\dx$.  These results beg the question: is the $\dx$ per-task burn-in for ERM fundamental or a technical byproduct? A $\dx$ burn-in is unintuitive, since given an optimal representation $\Phi_\star$, solving each task is precisely standard linear regression over $r$-dimensional covariates $z \triangleq \Phi_\star x$, for which the burn-in is much more lenient $\approx r$ \citep{wainwright2019high}. 

\ifshort{
\vspace{-0.3cm}
}

\paragraph{Non-linear multi-task learning: } 
Early works consider statistical guarantees for multi-task learning over general nonlinear function classes \citep{{baxter2000model, ben2008notion, maurer2016benefit, hanneke2022no}}; however, they do not obtain rates scaling jointly in $N,T$ due to the data model or assuming agnostic settings. \cite{du2020few, tripuraneni2020theory, watkins2024optimistic} provide excess risk bounds in which the error benefits jointly in $N$ and $T$ by assuming a \emph{shared representation}. 
\citet{du2020few} considers a setting in which $\calF$ is a linear function class, and $\calG$ is a nonlinear function class.  \citet{tripuraneni2020theory, watkins2024optimistic} consider nonlinear $\calF$ and $\calG$; however, the resulting generalization bounds scale with diameter of covariate distributions rather than with noise-level \citep{du2020few}.  
These works all assume marginal covariate distributions are identical across all tasks, and the final bounds involve data-dependent complexity terms. When instantiated in linear settings, their guarantees recover suboptimal burn-ins \textit{at least} order-$\dx$ samples per task. The aforementioned results study the ERM solution rather than feasible algorithms. \citet{meunier2023nonlinear} is a notable exception, providing a feasible algorithm in the setting of Reproducing Kernel Hilbert Spaces (RKHS), in which tasks share a RKHS subspace projection. 

\ifshort{
\vspace{-0.4cm}
}

\paragraph{Multi-task sequential learning: }

Multi-task learning has been applied to many dynamical systems settings, such as robotic manipulation \citep{brohan2022rt,shridhar2023perceiver} and agile flight \cite{o2022neural}. Despite its effectiveness in practice, the existing theoretical guarantees for representation learning do not apply to these settings due to the assumption that covariates are iid across tasks. Notably, when the predictor is either the dynamics function or a closed-loop control policy, the covariate distribution is inextricably linked with the predictor itself. Consider, for example, 
a stable autonomous system driven by white noise, $
            y_t \triangleq x_{t+1} = A x_t + w_t$ with $x_0,w_t \sim \calN(0,I_{\dx}).$
          The stationary covariate distribution is inextricably linked to the ``predictor'' $A$, as demonstrated by solving the Lyapunov equation \begin{align*}
              \Sigma_x &\triangleq \sfE[x_{t+1} x_{t+1}^\top] = A \sfE[x_t x_t^\top] A^\top + I_{\dx} =  A \Sigma_x A^\top + I_{\dx} \\
              \implies \Sigma_x &= \sum_{k \geq 0} A^k (A^k)^\top.
          \end{align*} Therefore, in multi-task settings where multiple distinct predictors are involved, the covariate distributions will be non-identical between tasks. Furthermore, covariates generated by dynamical systems are correlated across time.
These issues have been remedied in the linear setting by \citet{modi2021joint, zhang2023multi} applied to system identification and imitation learning by extending the analysis of \citet{du2020few} to consider covariates generated by linear systems. 
For the single-task non-linear regression setting,  \citet{ziemann2022learning,ziemann2023tutorial} demonstrate that learning in the presence of correlated covariates achieves the same rate as learning in the absence of correlation, only inflating the burn-in time by the correlation level. However, we are unaware of extensions of these results to multi-task representation learning.

In parallel, various works have considered extensions of the aforementioned multi-task linear regression works to linear bandit settings \citep{yang2022nearly, yang2020impact, du2023multi, mukherjee2023multi}. We also note works studying reinforcement learning (RL) with feature approximation (or low-rank Markov decision processes (MDPs) \citep{agarwal2020flambe, jin2020provably, uehara2021representation, efroni2022sample, du2019provably} which attain sample complexity gains from dimensionality reduction, but do not consider aggregating data across multiple tasks. Analysis of multi-task representation learning for MDPs has been studied by \citet{arora2020provable, lu2021power}; however these works assume generative models, thereby sidestepping the issues of independent data and non-identical covariates.
\ifshort{
\vspace{-0.3cm}
}

\subsection{Contributions}
\ifshort{
\vspace{-0.2cm}
}

In this work, we analyze the transfer learning problem in a setting where $\calF$ is a class of linear functions mapping $\mathbb{R}^{r}$ to $\R^{\dy}$, and $\calG$ is a class of nonlinear representations, as in \cite{du2020few}\footnote{This is a prototypical predictor model, e.g.\ in RL with feature approximation, nonlinear least squares, and classification.}. 
In this setting, we remove assumptions of both identical covariate distributions and independent covariates within tasks, and we additionally improve the per-task burn-in requirement.  We list our specific contributions:
    \begin{itemize}[noitemsep,nolistsep, leftmargin=*]
         \item We derive generalization bounds that hold for non-identical covariate distributions and vector-valued measurements. In particular, we present a refined ``task-diversity'' measure, which takes into account overlap of non-identical covariate distributions, in addition to the similarity of task-specific linear heads $\ftstar$ (e.g.\ \citet{du2020few, zhang2023multi}). 
     
        \item We show our proposed bounds on ERM scale multiplicatively with noise level and jointly with number of tasks and per-task samples as in \citet{du2020few}, while requiring only $\Omega(\dy r)$ samples per task for sufficiently large $T$ (noting $\dy = 1$ in most prior work, e.g.\ \citet{du2020few, tripuraneni2020theory}), as opposed to $\Omega(\dx)$. This illustrates the trend that as one increases the number of tasks, both the generalization error and sample requirement on each task converges to that of the $r$-dimensional regression of linear heads (as is the case if an optimal representation had been provided).

        \item We extend our bounds to (within-task) dependent data. Adapting ideas from recent work \citep{ziemann2022learning, ziemann2023noise}, we demonstrate that when task covariates are $\phi$-mixing, our generalization bounds scale with the independent-data rate. In particular, we avoid the effective sample-size deflation incurred by standard blocking techniques \citep{yu1994mixing, kuznetsov2017generalization}, relegating the effect of mixing to a mild increase of burn-in.
        
    \end{itemize}
    Notably, via our contributions, the guarantees in this work can be lifted from offline regression to various sequential decision-making settings, such as nonlinear system identification \citep{mania2022active, wagenmaker2023optimal} and stochastic contextual bandits \citep{foster2020beyond, simchi2022bypassing}. Stating our main theoretical result informally: 
    \begin{theorem}[Main result, informal]
        Let there be $T$ tasks and $N$ samples per task. Assume $N \geq C_{\mathrm{mix}}\blue{\Omega\paren{\dy r + \mathrm{C}(\calG)/T}}$, where $C_{\mathrm{mix}}$ characterizes the dependency of the covariates of each task. Then the excess transfer risk of ERM is bounded with high-probability:
        \begin{align*}
            \mathrm{ER}(\hat f^{(0)}, \hatg) \lesssim C_{\mathrm{task\;div}} \sigma^2 \paren{\frac{\mathrm{C}(\calG)}{NT} + \frac{\dy r}{N}},
        \end{align*}
        where $C_{\mathrm{task\;div}}$ characterizes the relatedness between the source tasks and the target task and $\sigma^2$ characterizes the level of the noise corrupting the measurements. 
    \end{theorem}

\paragraph{Notation}
Expectation (resp.\ probability) with respect to the underlying probability space is denoted by $\sfE$ (resp.\ $\sfP$). For two probability measures $\sfP$ and $\sfQ$ defined on the same probability space, their total variation is denoted $\|\sfP-\sfQ\|_{\mathsf{TV}}$. 
For an integer $n\in \N$, we also define the shorthand $[n] \triangleq\{1,\dots,n\}$.  The Euclidean norm on $\mathbb{R}^{d}$ is denoted $\|\cdot\|_2$,
and the unit sphere in $\R^d$ is denoted $\mathbb{S}^{d-1}$. We also write $\norm{M}_2$ for the spectral norm. We use $A^\dagger$ to denote the Moore-Penrose pseudo-inverse of $A$. For two symmetric matrices $M, N$, we write $M \succ N$ ($M\succeq N)$ if $M-N$ is positive (semi-)definite. 
We use $\lesssim, \gtrsim$ to omit universal numerical factors, and $\toP$ to denote convergence in probability.


\paragraph{Samples} In general, we index tasks by superscript while \textit{within-task} samples are indexed by subscript, e.g.\ $x_i^{(t)}$ for task $t$ and sample $i$. Let $\sfP_i^{(t)}, t\in [T]$ be probability measures over a fixed sample space $\mathsf{S}$. We are given $N$ samples from each ``training task'' $t$: $\sit \sim \sfP^{(t)}_{i}, t\in [T], i\in [N]$. For convenience, we overload notation and understand $\Pt$ alternatively refers to the stationary distribution when $\sit$ are identically distributed or to a joint \textit{trajectory} distribution $\scurly{\sit}_{i=1}^N \sim \Pt$ otherwise. Similarly, when we omit within-task indices $i$, we understand $\Et[f(S)] \triangleq \frac{1}{N}\sumN \Et[f(s_i^{(t)})]$. We use superscript $^{1:T}$ to denote a uniform mixture, e.g.\ $\sfP^{1:T} \triangleq \frac{1}{T} \sumT \Pt$. This work focuses on supervised learning: the sample space decomposes into an input (covariate) space $\sfX$ and and output (label) space $\sfY$: $\mathsf{S} = \sfX \times \sfY$ and we write $\sit = (\xit,\yit)$. Moreover, we are given $N'$ samples from a target task distributed according to a probability measure $\Pzero$ over $\sfZ$: $(\xiz, \yiz)\sim \sfP_i^{(0)}, i\in [N']$. It will also be convenient to introduce empirical counterparts $\hatPt, \hatEt$, such that e.g.\ $\hatEt[f(X)] = \frac{1}{N}\sumN f(\xit)$. We generally denote covariance matrices\footnote{To be precise, second moment matrices.} by $\Sigma$, e.g.\ $\Sigma_x^{(t)} \triangleq \Et [X X^\top]$.



\ifshort{
\vspace{-0.3cm}
}

\subsection{Problem Formulation}\label{sec: prob formulation}

Given the above definitions of training and test/transfer distributions, we consider a prototypical regression problem, where the goal of the learner is to perform well on the target task in terms of square loss over a fixed hypothesis class $\calH$. To enable transfer and characterize the benefits of representation learning, we assume that the hypothesis class $\calH$ under consideration splits into $\calH = \calF \times \calG$. We define the optimal training-task predictors:
\begin{align*}
    (\scurly{f_\star^{(t)}}_{t=1}^T,g_\star) \in \argmin_{\substack{(\scurly{f^{(t)}},g) \\ \in \calF^{\otimes T} \times \calG}}  \sum_{t=1}^T \Et \|f^{(t)}\mem \circ g(X) - Y\|_2^2.
\end{align*}
Hence, to each task $t\in[T]$ we associate a task-specific ``head" $f_{\star}^{(t)}\mem \in \calF$, while enforcing a shared ``representation'' $g_\star \in \calG$. We further denote the optimal target-task head: $f^{(0)}_{\star}\in \argmin_{f\in \calF}  \Ezero \|f\circ\mem g_\star(X)-Y\|_2^2$. 
Using our samples from both target and training tasks, we seek to find an element $(f,g) \in \calF\times \calG $ that renders 
the excess risk on the target distribution as small as possible:
\begin{align}
  &\ER^{(0)}(f,g) \label{eq:excess_risk} \\
  \triangleq \;&\Ezero \|f\circ\mem g(X)- Y\|_2^2 - \Ezero \|f_\star^{(0)}\circ\mem g_\star(X)-Y\|_2^2. \nonumber
\end{align}
In particular, we study the excess risk of a standard two-stage empirical risk minimization scheme \citep{du2020few,tripuraneni2020theory}, where a representation $\hatg \in \calG$ is fit on data from the $T$ training tasks, and a target-task head $\hat{f}^{(0)}$ is fit on the target task data, \textit{passed through} $\hatg$:
\begin{align}
        (\scurly{\hat{f}^{(t)}}_{t=1}^T,\hat{g}) &\in \argmin_{\calF^{\otimes T} \times \calG} \sum_{t=1}^T\sum_{i=1}^N \|f^{(t)}\mem \circ g(\xit) - \yit\|_2^2 \label{eq: first-stage ERM} \\
        \hat{f}^{(0)} &\in \argmin_{f\in\calF} \sum_{i=1}^{N'} \|f \circ \hatg(\xiz) - \yiz \|_2^2. \label{eq: two-stage ERM}
\end{align}
Though our work mostly concerns the statistical properties of ERM, we note that in many practical settings with expressive $\calG$, the empirical loss on a given dataset can be effectively optimized, and the error incurred by an algorithm enters as an additive factor in the generalization bounds \citep{vaskevicius2020statistical}. Toward characterizing bounds on the above excess risk, we consider vector-valued inputs and outputs $\sfX \times \sfY \subseteq \R^{\dx} \times \R^{\dy}$. We consider the \textit{realizable} setting, i.e.\ there exist $(\scurly{\ftstar}_{t=0}^{T}, g_\star)$ such that the noise term $W^{(t)} \triangleq 
Y^{(t)} - \ftstar\circ g_\star(X^{(t)})$ is a (conditionally) zero-mean process for every task. Importantly, we note $(\scurly{\ftstar}_{t=0}^{T}, g_\star)$ is generally not unique, e.g.\ if $\calF, \calG$ are both linear classes, $f_\star^{(t)}(z) = F^{(t)}_\star z \to F^{(t)}_\star Qz$ and $g_\star(x) = G_\star x \to Q^{-1} G_\star x$ remain optimal. The results in this paper should be understood to hold for \textit{any} tuple of optimal hypotheses $(\scurly{\ftstar}_{t=0}^{T}, g_\star)$.
\begin{assumption}\label{ass: conditional-subG noise}
    Given a filtration $\scurly{\scrF^{(t)}_{i}}_{i\geq 1}$ to which $\scurly{x_{i-1}^{(t)}}_{i\geq 1}$ is adapted, i.e.\ $\xit$ is predictable with respect to $\scrF^{(t)}_i$, for each $t \in [T]$, the noise sequence $\scurly{\wit}_{i\geq 1}$ is a $\sigma_{\sfW}^2$-\textit{conditionally subgaussian martingale difference sequence}:
    \begin{enumerate}[label=\alph*), itemsep=0pt]
        \item $\Et[\wit | \scrF^{(t)}_{i-1}] = 0$.
        \item $\Et[\exp(\lambda \langle\wit, v\rangle) \mid \scrF^{(t)}_{i-1}] \leq \exp(\lambda^2\sigma_{\sfW}^2/2)$, for all $\lambda \in \R$, $v \in \mathbb{S}^{\dy - 1}$, $i \geq 1$.
    \end{enumerate}
\end{assumption}
\Cref{ass: conditional-subG noise} simply asserts the noise is independent, zero-mean subgaussian for independent data, with the additional formalism necessary when extending to sequentially dependent settings.

%% file: icml/icml_tex/theory_results.tex
\vspace{-0.3cm}
\section{Main Results}\label{sec: main results}

In this section, we present our main results and the key steps in the proof. Firstly, we present the main definitions and assumptions in \Cref{sec: canonical decomposition}, and convert the target-task excess risk to quantities defined over the training tasks. We then instantiate in \Cref{sec: main; indep cov finite class} a basic setting where $\calG$ is finite and within-task samples are iid, but task-wise covariate distributions may be non-identical, $\Ptx \neq \sfP^{(t')}_{\sfX}$, in order to highlight the benefits brought by our analysis. 
In \Cref{sec: main; rep learning with little mixing}, we lift our results to general representations $\calG$ and settings where within-task samples may be sequentially dependent. In particular, we leverage recent literature to shift the effect of dependency to the burn-in, resulting in rates analogous to the independent data setting.

\vspace{-0.25cm}
\subsection{Task Diversity and A Canonical Decomposition}\label{sec: canonical decomposition}

A non-vacuous bound on the excess transfer risk is only possible if the source tasks are somehow informative for the target task. Therefore, a pervasive step in establishing a bound lies in relating the risk on the target task to the average risk over the training tasks, where the quality of this relation is determined by a ``task-diversity'' condition. To make this concrete, we adapt such a condition from \cite{tripuraneni2020theory}.
\begin{definition}[Task-Diversity \citep{tripuraneni2020theory}]\label{ass: id}
The training tasks satisfy a task-diversity condition at level $\nu >0$, if for any $g \in \calG$ the following holds:
\ifshort{
\begin{align}\label{eq: nu task diversity}\tag{\texttt{TD}}
    \begin{split}
        &\inf_{f\in \calF} \ER(f,g)
        \\
        \leq \;& \frac{\nu^{-1}}{ T }\sumT\inf_{f^{(t)}}\; \Et \|f^{(t)}\circ g(X)- \ftstar \circ g_\star(X)\|_2^2
    \end{split}
\end{align}
}
{
\begin{align}\label{eq: nu task diversity}\tag{\texttt{TD}}
    \begin{split}
        \inf_{f\in \calF} \ER^{(0)}(f,g) &\leq \frac{\nu^{-1}}{ T }\sumT \inf_{f\in \calF} \ER^{(t)}(f,g)\\
        &= \frac{\nu^{-1}}{ T }\sumT\inf_{f^{(t)}}\; \Et \|f^{(t)}\circ g(X)- \ftstar \circ g_\star(X)\|_2^2.
    \end{split}
\end{align}
}
\end{definition}
Intuitively, \eqref{eq: nu task diversity} measures to what degree generalization on-average across source tasks \textit{certifies} generalization on the target task.\footnote{We note that that $\nu$ must hold given any $g$, but $f$ is assumed to be optimized per-task. This is weaker than $\nu$ holding over every $f^{(0)}, f^{(1)}, \dots, f^{(T)} \in \calF$, $g \in \calG$, and is better fit for the two-stage ERM framework.} We then consider a trivial canonical risk decomposition:
\ifshort{
\begin{align*}
    &\ER(f,g) \\
   =\;& \Ezero \norm{f\mem \circ g(X) - Y}_2^2  - \inf_{f' \in \calF} \Ezero\norm{f'\mem\circ g(X) - Y}_2^2  \\
   +\;&\inf_{f' \in \calF} \Ezero\norm{f'\mem\circ g(X) - Y}_2^2 - \Ezero \|f_\star^{(0)}\mem \circ g_\star(X) - Y\|_2^2.
\end{align*}
}
{
\begin{align*}
    \ER^{(0)}(f,g) &= \Ezero \norm{f\mem \circ g(X) - Y}_2^2  - \inf_{f' \in \calF} \Ezero\norm{f'\mem\circ g(X) - Y}_2^2  \\
   &\quad+ \inf_{f' \in \calF} \Ezero\norm{f'\mem\circ g(X) - Y}_2^2 - \Ezero \|f_\star^{(0)}\mem \circ g_\star(X) - Y\|_2^2.
\end{align*}
}
Applying the task-diversity condition \eqref{eq: nu task diversity} to the last line, and observing any plug-in $f^{(t)}$ upper bounds each infimum (in particular $f^{(t)} = \fthat$) yields the following result.
\begin{lemma}\label{lemma: excess risk two term decomp}
    Let $(\scurly{\fthat}_{t=0}^{T}, \hat g)$ be the output of the two-stage ERM \eqref{eq: two-stage ERM}. Assuming the task-diversity condition \eqref{ass: id} holds at level $\nu > 0$, then
    \ifshort{
    \begin{align}
        &\hspace{-0.3cm}\ER(\hat{f}^{(0)},\hatg) \nonumber \leq \\
        &\hspace{-0.3cm}\overset{(\text{Risk of non-realizable regression on targets }\hat g(X))}{\Ezero \|\hat{f}^{(0)}\mem \circ \hat g(X) - Y\|_2^2  - \inf_{f'\mem \in \calF} \Ezero\snorm{f'\circ \hat g(X) - Y}_2^2} \label{eq: non-realizable regression}\\
        &\hspace{-0.3cm}+ \overset{(\text{Task-average estimation error of training task predictors})}{\frac{\nu^{-1}}{T} \sumT \Et \snorm{\hat{f}^{(t)}\mem \circ \hat g(X) - \ftstar\mem \circ g_\star(X)}_2^2.} \label{eq: task-averaged est error}
    \end{align}
    }
    {
    \begin{align}
        \ER^{(0)}(\hat{f}^{(0)},\hatg) &\leq \overset{(\text{Risk of non-realizable regression on targets }\hat g(X))}{\Ezero \|\hat{f}^{(0)}\mem \circ \hat g(X) - Y\|_2^2  - \inf_{f'\mem \in \calF} \Ezero\snorm{f'\circ \hat g(X) - Y}_2^2} \label{eq: non-realizable regression}\\
        &\quad + \overset{(\text{Task-averaged estimation error of training task predictors})}{\frac{\nu^{-1}}{T} \sumT \Et \snorm{\hat{f}^{(t)}\mem \circ \hat g(X) - \ftstar\mem \circ g_\star(X)}_2^2.} \label{eq: task-averaged est error}
    \end{align}
    }
\end{lemma}
As outlined above, the risk of the transfer task predictor can be bounded by two main terms. The former term is precisely the excess risk of target-task head $\hat f^{(0)}$ over the $r$-dimensional inputs $\hat g(X)$; since $\hat g$ via the two-stage ERM is statistically independent of $\Pzero$, it may be treated as fixed. Since generally $\hat g \neq g_\star$, regressing $Y$ against $\hat g(X)$ is \textit{non-realizable}. In particular, this breaks the (conditional) independence between the error $u_i = y_i - f \circ \hat g(x_i)$ and covariate $x_i$. The latter term \eqref{eq: task-averaged est error} is the population task-averaged estimation error of the ERM predictors $\fthat \circ \hat g$ with respect to optimal $\ftstar \circ g_\star$, adjusted by task diversity parameter $\nu$.

\Cref{lemma: excess risk two term decomp} and definitions therein thus far hold for general composite classes $\calF \times \calG$. Toward substantiating bounds on \eqref{eq: non-realizable regression} and \eqref{eq: task-averaged est error}, we consider a prevalent model of non-linear representation learning, where $\calG$ is an arbitrary function class that embeds the inputs into a low-dimensional latent space in $\R^r$, and the task-specific heads act linearly on $\R^r$ \citep{du2020few, meunier2023nonlinear, collins2023provable}. Besides being an established theoretical model, we mention that last-layer finetuning (alternatively known as linear probing) has empirical benefits for multi-task / out-of-distribution transfer compared to fine-tuning the full model \citep{kumar2022fine, lee2023surgical}, and that more complex task-specification classes $\calF$ can lead to provable and empirical difficulties extracting task diversity \citep{xu2021representation}.
\begin{assumption}[Low-dim.\ representations]\label{ass: low-dim reps}
    The representation class $\calG$ embeds inputs to $\R^r$, $g : \R^{\dx} \to \R^{r} \;\forall g \in \calG$, where $r \leq \dx$. The task-specific head class $\calF$ is linear: $\calF = \scurly{f: f(z) = Fz,\;F \in \R^{\dy \times r}}$. 
\end{assumption}
In particular, \Cref{ass: low-dim reps} turns bounding \eqref{eq: non-realizable regression} into bounding the excess risk of a non-realizable least squares problem \citep{ziemann2023noise}.
We now discuss sufficient conditions to quantify the task-diversity parameter $\nu$. We define the following ``task-coverage'' condition.
    
\begin{definition}\label{def: mu task coverage}
    For a given $g \in \calG$, define the stacked covariance matrices and the corresponding Schur complements for each $t = 0,\dots, T$:
    \begin{align*}
        \Sigma_g^{(t)} &\triangleq \Et \brac{\bmat{g(X) \\ g_\star(X)} \bmat{g(X) \\ g_\star(X)}^\top} \\
        \overline\Sigma_g^{(t)} &\triangleq \Et\brac{g_\star(X) g_\star(X)^\top} - \Et\brac{g_\star(X) g(X)^\top} \Et\brac{g(X) g(X)^\top}^{-1} \Et\brac{g(X) g_\star(X)^\top} 
    \end{align*}
    We say the \emph{task-coverage condition} holds if there exists a constant $\mu_{\sfX} > 0$ such that for all $g \in \calG$:
    \begin{align}\label{eq: mu task coverage}\tag{\texttt{TC}}
        \overline \Sigma_g^{(0)} \preceq \mu_{\sfX} \overline\Sigma_g^{(t)}.
    \end{align}
    This smallest such $\mu_{\sfX}$ may be defined as
    \begin{align*}
        \mu_{\sfX} \triangleq \max_{g \in \calG} \max_{t \in [T]} \snorm{(\overline\Sigma_g^{(t)})^{\dagger/2} \overline\Sigma_g^{(0)}(\overline\Sigma_g^{(t)})^{\dagger/2}}_2.
    \end{align*}
\end{definition}
\vspace{-0.2cm}
Since Schur complements preserve psd ordering, \Cref{def: mu task coverage} is immediately implied by psd dominance of the covariance matrices $\Sigma_g^{(0)} \preceq \mu_{\sfX} \Sigma_g^{(t)}$, $t \in [T]$. Furthermore, we may relax the maximum over all $t$ to any constant-fraction subset of $[T]$, or more intricate notions that also weight the contribution of each task in \eqref{eq: nu task diversity} beyond the uniform $1/T$. However, for conciseness we do not discuss these extensions. Intuitively, \Cref{def: mu task coverage} quantifies the degree to which each training task covariate distribution covers the target covariate distribution, passed through the representation class. Larger $\mu_{\sfX}$ implies ``worse'' coverage, affecting the transfer learning rate.
\begin{remark}\label{remark: examples of mu task coverage}
    When covariates are identically distributed for all tasks $\Ptx = \sfP^{(t')}_\sfX$, $t,t' \in \scurly{0,\dots, T}$, then $\mu_{\sfX}$-\eqref{eq: mu task coverage} holds with equality and $\mu_{\sfX} = 1$.
    When $\calG$ is a linear class $g(x) = Gx$, then $\mu_{\sfX}$-\eqref{eq: mu task coverage} holds for any $\mu_{\sfX}$ such that $\Sigma_{\sfX}^{(0)} \preceq \mu_{\sfX} \Sigma_{\sfX}^{(t)}$, $t \in [T]$, which follows from combining $\Sigma^{(t)}_g = \bmat{G \\ G_\star} \Sigma_{\sfX}^{(t)} \bmat{G \\ G_\star}^\top$ with the fact $P \preceq Q$ implies $M P M^\top \preceq M Q M^\top$ for any $M$. Notably, this recovers the ``$c$'' parameter in \citet[Assumption 4.2]{du2020few} $c\Sigma_{\sfX}^{(0)} \preceq \Sigma_{\sfX}^{(t)}\;\forall t\in[T]$.
\end{remark}
One of our key results demonstrates our notion of task-coverage implies a bound on the task-diversity parameter, captured in the following result.
\begin{restatable}[\eqref{eq: mu task coverage}$\implies$\eqref{eq: nu task diversity}]{proposition}{TaskCoverageTaskDiversity}\label{prop: task coverage -> task diversity}
    Let \Cref{ass: low-dim reps} hold. Define $\bfF_\star^{(0)} \triangleq \Fzerostar^\top \Fzerostar$ and $\bfF^{1:T}_\star \triangleq \frac{1}{T}\sumT \Ftstar^\top \Ftstar \in \R^{r \times r}$, and suppose $\range(\bfF_\star^{(0)}) \subseteq \range(\bfF^{1:T}_\star)$. Define the head-coverage coefficient
    \begin{align}\label{eq: head coverage}
        \mu_{F} \triangleq \snorm{(\bfF^{1:T}_\star)^{\dagger/2} \bfF_\star^{(0)}(\bfF^{1:T}_\star)^{\dagger/2}}_2.
    \end{align}
    Then any problem instance satisfying $\mu_{\sfX}$-\eqref{eq: mu task coverage} also satisfies $\nu$-\eqref{eq: nu task diversity} with $\nu^{-1} = \mu_{\sfX} \mu_{F}$.
\end{restatable}

The proof is found in \Cref{appendix: main results}. It may be helpful to consider scalar outputs $\dy = 1$, where the range requirement of \Cref{prop: task coverage -> task diversity} is equivalent to $\Fzerostar \in \Span(F^{(1)}_\star, \dots, F^{(T)}_\star)$. If this is not satisfied, then in the worst case $\sfP^{(0)}_{\sfX}$ may only hit the component $\Fzerostar$ orthogonal to the span, for which the training data is uninformative. We also note that $\mu_F = \snorm{(\bfF^{1:T}_\star)^{\dagger/2} \bfF_\star^{(0)}(\bfF^{1:T}_\star)^{\dagger/2}}_2$ is precisely the generalization proposed in \cite{zhang2023multi} of the ``task-diversity parameter'' \citep{du2020few, tripuraneni2021provable}. Therein, task-diversity is certified by directly assuming normalization and well-conditioning $\lambda_i(\bfF^{1:T}_\star) = \Theta(T/r)$, $i=1,\dots, r$. Generality aside, this also accrues an additional factor of $r$ in the final rates when the eigenvalues of $\bfF_\star^{(0)}$ non-uniform (see \citet[Remark 4.2]{du2020few}). 

\begin{remark}[Robustness to overspecified $r$]
    Another consequence of uniformity assumptions on the training task heads $\bfF^{1:T}$ is that the representation dimension $r$ must in general be exactly specified: underspecification (i.e.\ $r$ is set lower than necessary) leads to non-realizability, and overspecification in general implies $\bfF^{1:T}_\star$ may not be full-rank, let alone have approximately uniform eigenvalues. Since in practice the representation dimension is often a user-specified parameter (e.g.\ hidden dimension of a neural net), it is important that a certificate of task-diversity only considers the range of optimal training-task heads, as in \Cref{prop: task coverage -> task diversity}.
\end{remark}

Compared to prior work, we suggest that task diversity should actually be measured by the joint quantity $\mu_{\sfX} \mu_F$. Whereas $\mu_F$ is precisely $\nu^{-1}$ when task covariates are identical (\Cref{remark: examples of mu task coverage}), in general the alignment of the train and target task heads \textit{and} of the covariate distributions both contribute to task diversity. Pathologically, if the train and target task covariates have disjoint supports, even if the heads are identical $\Fzerostar = \Ftstar,\;\forall t \in [T]$ ($\mu_F = 1$), the error induced by a given $(F,g)$ on the training distributions is in general uninformative to that on the target distribution. Similarly, non-trivial transfer risk is generally impossible when $\range (\bfF_\star^{(0)}) \not\subseteq \range (\bfF^{1:T}_\star)$, even when $\sfP^{(0)}_{\sfX} = \sfP^{(t)}_{\sfX}$, $\forall t \in [T]$.

\subsubsection*{An Excursion: Directly Estimating $\nu$}

We have demonstrated how task-diversity (\Cref{ass: id}) can be leveraged to bound the excess risk on the transfer task by the training-task-averaged risk. Furthermore, we demonstrated how the abstract task diversity coefficient $\nu$ can be bounded in a problem-independent manner by measures of covariate and head coverage $\mu_{\sfX}, \mu_{F}$ that illustrate the scaling of task diversity as tasks deviate from identicality. However, if the goal is to numerically \textit{estimate} the task-diversity of a set of tasks $t = 0,\dots, T$ from data, then we demonstrate that this is possible in our setting, even though the excess risks $\ER^{(t)}$ are generally unsavory to directly estimate. We note that $\nu$ is defined as a uniform quantity over $\calG$, which is generally conservative and intractable to search over. We therefore consider estimating the task-diversity induced by a \textit{given} representation $g \in \calG$:
\begin{align}\label{eq: rep-conditioned task-diversity}
    \nu(g) \triangleq \paren{\frac{1}{T}\sumT \inf_{f\in \calF} \ER^{(t)}(f,g)}/\inf_{f\in \calF} \ER^{(0)}(f,g).
\end{align}
By this definition $\nu(g)$ is a measure of how ``informative'' the pretraining tasks $t =1,\dots,T$ are for certifying the excess risk of a given representation $g$ for the target task, assuming that an optimal head $F^{(t)}$ had been fit on each task.

\begin{definition}\label{def: task div estimator}
    Given $g \in \calG$ and a batch of data $\scurly{(\xit, \yit)}_{i=1,t=0}^{N, T}$, define the $g$-conditioned empirical least squares heads:
    \begin{align*}
        \widehat F_{g}^{(t)} \triangleq \argmin_{F \in \calF} \;\hatEt \brac{\snorm{Y - F g(X)}^2}, \quad t = 0,\dots, T.
    \end{align*}
    We define the following estimator of $\nu(g)$:
    \begin{align}\label{eq: task div estimator}
    \hat\nu_N(g) \triangleq \frac{\frac{1}{T}\sumT \hatEt[\snorm{Y}_2^2] - \trace\paren{\widehat F_{g}^{(t)} \hatEt[g(X)g(X)^\top] \mathopen{}\widehat F_{g}^{(t)}\mathclose{}^\top}}{ \hatEzero \brac{\snorm{Y}_2^2} - \trace\paren{\widehat F_{g}^{(0)} \hatEzero[g(X)g(X)^\top] \mathopen{}\widehat F_{g}^{(0)}\mathclose{}^\top} }.
    \end{align}
\end{definition}
We note that $\hat\nu_N(g)$ does not require any additional information beyond the data batch and the given representation $g$. It turns out this simple estimator can be shown to be consistent.
\begin{restatable}[Convergence of $\hat\nu_N(g)$]{proposition}{taskdivconsistency}\label{prop: consistency of nu(g)}
    Let \Cref{ass: conditional-subG noise} and \Cref{ass: low-dim reps} hold. For convenience, assume $\wit = 0$ for all $i, t$, and $\xit$ are iid across $i$ for each $t = 0, \dots, T$.
    Then, the empirical estimate $\hat \nu_N (g)$ is consistent: $\hat\nu_N(g) \toP \nu(g)$.
\end{restatable}
The proof of \Cref{prop: consistency of nu(g)} can be found in \Cref{appdx: est task diversity}. We note that convergence of $\hat\nu_N(g)$ can likely be refined for finite-sample guarantees; we leave this to future work. When noise $\wit$ is present and we depart from independence, we expect $\hat\nu(g)$ (or most other estimators) to be biased. However, for low-noise settings, we note if both the numerator and denominator of \eqref{eq: task div estimator} are small (i.e.\ close to the noise-level), this means that $g$ is already close to optimal. In summary, in addition to being an object for capturing the theoretical benefit of representation learning, this demonstrates the potential for task-diversity to be utilized as a \textit{data-dependent estimate} of task-relevance, with potential applications in e.g.\, adaptive sampling and active learning \citep{chen2022active, wang2023improved}.

\begin{remark}
    The utility of $\nu(g)$ implicitly depends on an assumption of (algorithmic) stability \citep{bousquet2002stability}. For example, letting $\hat g_N$ be the output of the first-stage ERM \eqref{eq: first-stage ERM} with $N$ datapoints per task, we might hope that $\nu(\hat g_N) \toP V_\star \subseteq (0, \infty)$, such that the task diversity of a given draw of $\hat g_N$ for a given $N$ is informative. Notably, $\nu(g_\star)$ is vacuous, since both the LHS and RHS of \eqref{eq: nu task diversity} are $0$, and thus bounding the limit of $\nu(\hat g_N)$ \textit{a priori} is non-trivial. We leave exploring the stability of task diversity as an interesting direction for future work.
\end{remark}


\subsection{Warm-Up: Independent Covariates and Finite $\calG$}\label{sec: main; indep cov finite class}

In this section, we consider a basic setting where covariates are iid within-task (possibly non-identical between tasks) and where the representation class $\calG$ is finite for simplicity. We now identify how the ideas introduced in the prequel lead to sample-efficient guarantees for representation learning. As previewed earlier, the target excess risk induced by the ERM $(\Fzerohat, \hatg)$ amounts to bounding two separate terms--the excess risk of a non-realizable least-squares regression, and the task-average estimation error of the ERM training predictors $(\scurly{\Fthat}_{t=1}^{T}, \hatg)$. We make the following boundedness assumptions to simplify ensuing expressions.
\begin{assumption}\label{ass: function class boundedness}
    Let $\calF \subseteq \scurly{F \in \R^{\dy \times r}: \norm{F}_F \leq B_\calF}$, and $\sup_{g\in \calG} \sup_{x \in \calX} \norm{g(x)}_2 \leq B_\calG$.
\end{assumption}
Lastly, defining the centered function class $\overline{\calH} \triangleq \calF^{\otimes T}\mem \times G - \scurly{\Ftstar}_{t=1}^T\mem \times \scurly{g_\star}$, the following is an adaptation of a standard assumption \citep{liang2015learning, oliveira2016lower, koltchinskii2015bounding, ziemann2022learning}, \citet[Chapter 14.2]{wainwright2019high}.
\begin{assumption}[Hypercontractivity]\label{ass: hypercontractivity}
    We assume $(\overline{\calH}, \sfP^{1:T})$ and $(\overline{\calH}, \Pzero)$ satisfy ($4$-$2$) hypercontractivity: for each $\overline h \in \overline\calH$,
    \begin{align}
        &\sfE^{1:T} \snorm{\overline h(X)}_2^4 \leq C^{1:T}_{4\to2}\paren{\sfE^{1:T} \snorm{\overline h(X)}_2^2}^2, \\
        &\Ezero \snorm{\overline h(X)}_2^4 \leq C^{(0)}_{4\to2}\paren{\Ezero \snorm{\overline h(X)}_2^2}^2.
    \end{align}
\end{assumption}
We note that the choice of $4\to 2$ is rather arbitrary and can be substituted for, e.g.\ $2p \to 2$, $p \geq 2$ moment equivalence, making the appropriate adjustments to various terms downstream.  Examples of hypercontractivity can be found in, e.g.\ \cite{ziemann2022learning}.

\ifshort{
\paragraph{Bounding Nonrealizable Least-Squares Error}
}
{
\subsubsection*{Bounding Nonrealizable Least-Squares Error}
}

Given the representation $\hat g$ outputted by the training phase of the two-stage ERM \eqref{eq: first-stage ERM}, let us define the random variable $Z \triangleq \hat g(X) \in \R^{r}$, and a best-in-class (misspecified) linear head on $Z$ as
\[
    \Fzerohatstar \triangleq \argmin_{F \in \R^{\dy \times r}} \Ezero \snorm{Y - FZ}_2^2
\]
Since $\hat g$ is fixed with respect to $\Pzero$, we may re-write \eqref{eq: non-realizable regression} as
\begin{align}
\begin{split}\label{eq: NRLS excess risk as frob norm error}
    &\Ezero \snorm{\Fzerohat Z - Y}_2^2 - \Ezero \snorm{\Fzerohatstar Z - Y}_2^2 \\
    =\;& \snorm{(\Fzerohat - \Fzerohatstar)\sqrt{\Sigmaz_Z}}_F^2, \quad \Sigmaz_Z \triangleq \Ezero[ZZ^\top].
\end{split}
\end{align}
Define the (possibly biased) noise variable $U \triangleq Y - \Fzerohatstar Z$. By the two-stage ERM, $\Fzerohat$ is precisely the least-squares solution on datapoints $\scurly{(\ziz, \yiz)}_{i=1}^{N'}$. Therefore, we may adapt results from \cite{oliveira2016lower} and \cite{ziemann2023noise} to bound the excess risk \eqref{eq: NRLS excess risk as frob norm error}. Following these works, we introduce the following quantities.
\begin{definition}\label{def: NRLS moment-equiv quants}
    Define the noise-class interaction term $V \triangleq UZ^\top \mathopen{}\Sigmaz_Z\mathclose{}^{-1/2}$. We define the following quantities:
    \begin{align*}
        \hhatZ^2 &\triangleq \max_{v:\; v^\top \Sigmaz_Z v = 1}\Ezero[\ip{v, Z}^4], \quad \hhatV \triangleq \frac{\norm{\snorm{V}_F}^2_{\Psi_1}}{\Ezero[\snorm{V}_F^2]},
    \end{align*}
    where $\snorm{X}_{\Psi_m} \triangleq \sup_{p \geq 1} p^{-1/m} \snorm{X}_{\calL^p}$.
\end{definition}
We note that in our problem set-up, these quantities are guaranteed to exist for a fixed representation $\hat g$. We discuss immediate upper bounds on these quantities in \Cref{appdx: non-realizable least squares}. Whether or not one can extract global bounds over $\calG$ is a subtle question and is fundamentally tied to the nature of ERM. For example, even in the linear representation setting, one can pick an ERM $\hat G$ to render the second-stage least-squares problem \eqref{eq: two-stage ERM} arbitrarily ill-conditioned by the equivalence structure $F \to F Q,\; G \to  Q^{-1}G$, and thus all ERM guarantees therein \citep{du2020few, tripuraneni2021provable, zhang2023multi} actually perform ERM over a quotient set of $\calG$ (e.g.\ row-orthonormal $G \in \R^{r \times \dx}$). Regardless, these quantities do not explicitly depend on the problem dimension and appear only in the burn-in of the ensuing guarantee for our non-realizable least-squares problem.\footnote{As a sanity check, $\hhatZ, \hhatV$ are bounded by absolute constants when $\calG$ is linear and $X,W$ are Gaussian.}
\begin{restatable}{proposition}{iidNRLSbound}\label{prop: iid NRLS bound}
    Fix $\delta \in (0,1/e)$. Define $\sigma_U^2 \triangleq \sqrt{\Ezero\brac{\snorm{U}_2^4}}$, $\sigma_V^2 \triangleq \Ezero\brac{\snorm{V}_F^2}$ and $C_{Z} \triangleq \sup_{v\in \mathbb{S}^{\dy-1}} \sqrt{\Ezero \langle v, \Sigmaz_Z\mathclose{}^{-1/2} Z \rangle^4 }$. Let $\hhatZ, \hhatV$ be as defined in \Cref{def: NRLS moment-equiv quants}. As long as the burn-in conditions hold:
    \begin{align*}
        N' \gtrsim r + \hhatZ^2 \log(1/\delta), \quad N' \gtrsim \hhatV^2 \paren{\frac{\log(1/\delta)}{\log(N')}}^{8},
    \end{align*}
    then with probability at least $1 - \delta$ we have
    \begin{align*}
        \snorm{(\Fzerohat - \Fzerohatstar)\sqrt{\Sigmaz_Z}}_F^2 &\lesssim \frac{\sigma_V^2 \log(1/\delta)}{N'} \\
        &\lesssim \frac{C_Z \sigma_U^2  r \log(1/\delta)}{N'}.
    \end{align*}
    
\end{restatable}

In \Cref{prop: iid NRLS bound}, we express the excess risk of the non-realizable least squares in terms of the variance proxy $\sigma_U^2$. We shall now relate $\sigma_U^2$ to $\sigma_{\sfW}^2$, the ``noise-level'' of the underlying data-generating process.
To reason about the magnitude of this quantity, we may re-arrange $\nu$-\eqref{eq: nu task diversity} to yield the following lemma.
\begin{restatable}{lemma}{noiselevellemma}\label{lem: noiselevellemma}
Let $\sigma_U^2$ be defined as in \Cref{prop: iid NRLS bound}. Then:
\ifshort{
    \begin{align*}
    \sigma_U^2 &\lesssim  \dy\sigma^2_W 
    \\
    &+\frac{\sqrt{C^{(0)}_{4\to 2}}\nu^{-1}}{T} \sum_{t=1}^T \Et \snorm{\Fthat \hat g(X) - \Ftstar g_\star (X)}_2^2.
\end{align*}
}
{
    \begin{align*}
    \sigma_U^2 &\lesssim  \dy\sigma^2_W +\frac{\sqrt{C^{(0)}_{4\to 2}}\nu^{-1}}{T} \sum_{t=1}^T \Et \snorm{\Fthat \hat g(X) - \Ftstar g_\star (X)}_2^2.
\end{align*}
}
\end{restatable}



In other words, the noise level of the misspecified model is no more than the optimal noise level plus the familiar task-averaged estimation error \eqref{eq: task-averaged est error}, which we note is divided by an additional factor of $N'$ in \Cref{prop: iid NRLS bound}. Therefore, we have isolated the task-averaged estimation error \eqref{eq: task-averaged est error} as the sole remaining quantity to control.

\ifshort{
\paragraph{Bounding Task-Averaged Estimation Error}
}
{
\subsubsection*{Bounding Task-Averaged Estimation Error}
}

As mentioned above, the last remaining task is to control the task-averaged estimation error. As previously discussed, the key observation is to quantify a lower uniform law, such that, roughly speaking, the empirical estimation error dominates the population counterpart:
\begin{align*}
    \sfE^{1:T} \norm{\overline h}_2^2 \lesssim \widehat{\sfE}^{1:T} \norm{\overline h}_2^2, \text{ for all }\overline h \in \overline\calH.
\end{align*}
Toward this end, we show that hypercontractivity (\Cref{ass: hypercontractivity}) leads to a lower estimate for any \textit{given} $\overline h \in \overline \calH$ (\Cref{prop: lower isometry for given fxn}).
By an application of the \textit{offset basic inequality} \citep{rakhlin2014online, liang2015learning}, an empirical estimation error can be bounded by
\ifshort{
\begin{align}
    &\frac{1}{NT} \sumT\sumN \snorm{h(\xit)}_2^2 \nonumber\\
    \leq\;& \sup_{h \in \calH} \frac{1}{NT} \sumT\sumN 4\ip{\wit, h(\xit)} - \snorm{h(\xit)}_2^2\nonumber \\
    \triangleq\;& \sfM_{NT}(\calH),\nonumber
\end{align}
}
{
\begin{align}
    \frac{1}{NT} \sumT\sumN \snorm{h(\xit)}_2^2 \leq& \sup_{h \in \calH} \frac{1}{NT} \sumT\sumN 4\ip{\wit, h(\xit)} - \snorm{h(\xit)}_2^2\nonumber \\
    \triangleq&\; \sfM_{NT}(\calH),\nonumber
\end{align}
}
where $\sfM_{NT}(\calH)$ is denoted the (empirical) \textit{martingale offset complexity} \citep{liang2015learning, ziemann2022learning}, which serves as the capacity measure of a hypothesis class $\calH$. Notably, $\sfM_{NT}(\calH)$ scales with the \textit{noise-level} $\sigma_{\sfW}^2$, rather than the diameter of $\calH$. Via a high-probability chaining bound (\Cref{lem: high probability chaining}), we demonstrate $\sfM_{NT}(\calH)$ is controlled by a log-covering number of $\calH$ at a resolution $\gamma$ of our choice. As a result, there is a regime of $\gamma$ such that with probability at least $1 - \delta$,
\vspace{-0.2cm}
\begin{align*}
    \sfM_{NT}(\calH) &\lesssim \frac{\sigma_{\sfW}^2}{NT}\paren{\log \sfN_\infty(\calH, \gamma) + \log(1/\delta)},
\end{align*}
where $\sfN_\infty(\calH, \gamma)$ is the covering number of $\calH$ in the supremum metric: $\rho(h_1, h_2) = \sup_{x \in \sfX} \norm{h_1(x) - h_2(x)}_2$. For salient choices of $\gamma$, we want $\sfM_{NT}(\calH)$ to be the dominant scaling in the estimation error bound. We then proceed to a localization argument, where we can define disjoint events over elements of $\overline \calH$ (strictly speaking, over a class that subsumes $\overline\calH$) -- either: 1.\ the population estimation error is within an $\tau^2$ radius around zero, or 2.\ the estimation error exceeds $\tau^2$ but is dominated by the empirical error, which is bounded by the martingale offset complexity. In particular, the probability of neither event holding can be controlled by union bounding over a finite $\calO(\tau)$-cover of $\overline\calH$, such that we have with probability at least $1 - p(\tau, N, T)$, for all $\overline h \in \overline\calH$:
\begin{align}
    \sfE^{1:T}\snorm{\overline h}_2^2 \lesssim \max\scurly{\sfM_{NT}(\overline\calH), \tau^2}.
\end{align}
Therefore, this informs choosing $\gamma, \tau$ such that the two terms meet at the desired rate. The failure probability $p(\tau, N, T)$ turns into a burn-in condition on $N, T$ when inverted for $\delta$. As the last step before bounding the estimation error, we note that $\calF^{\otimes T}$ can be identified with a bounded set in $\R^{T\dy r}$, and therefore we get the following straightforward bound on the covering number
\begin{align*}
    \log \sfN_{\infty}(\overline\calH, \varepsilon) \leq T\dy r \log\paren{1 + \frac{4 B_\calF B_\calG}{\varepsilon}} + \log\abs{\calG}.
\end{align*}
The aforementioned steps and proofs are found in \Cref{lem: high probability chaining}, \Cref{prop: localization}, and \Cref{lem: log covering bound star hull}. Optimizing resolutions $\gamma$ and $\tau$ yields a bound the task-averaged estimation error.
\begin{restatable}{proposition}{iidEstErrorBound}\label{prop: iid estimation error bound}
    Let \Cref{ass: function class boundedness} and let $C^{1:T}_{4 \to 2}$ be defined as in \Cref{ass: hypercontractivity}. Then, with probability at least $1 - \delta$, the estimation error of ERM predictors $\scurly{\Fthat}_{t=1}^T, \hat g$ is bounded by
\ifshort{
\begin{align*}
    &\frac{1}{T}\sumT \Et\snorm{\Fthat \hat g(X) - \Ftstar g_\star(X)}_2^2 \\
    &\leq \sigma_{\sfW}^2 \cdot \tilde\calO\paren{\frac{\dy r}{N} + \frac{\log\abs{\calG} + \log(1/\delta)}{NT}},
\end{align*}
}
{
\begin{align*}
    \frac{1}{T}\sumT \Et\snorm{\Fthat \hat g(X) - \Ftstar g_\star(X)}_2^2 \lesssim \sigma_\sfW^2 \Bigg(\frac{\dy r}{N}\log\paren{\frac{B_\calF B_\calG NT}{\sigma_\sfW} } + \blue{\frac{\log\abs{\calG} + \log(1/\delta)}{NT}}  \Bigg),
\end{align*}
}
as long as $N\gtrsim C^{1:T}_{4 \to 2} \paren{ \dy r \log\paren{\frac{B_\calF B_\calG NT}{\sigma_\sfW}} +  \blue{\frac{\log\abs{\calG} + \log(1/\delta)}{T}} }$.
\end{restatable} 
We note that \Cref{prop: iid estimation error bound} exhibits two critical benefits of multi-task representation learning; namely, the complexity term associated with the representation class is divided by the total number of tasks in both the rate \textit{and} the burn-in requirement. This properly describes the qualitative trend we hope to see: when the number of tasks is large (and thus $g_\star$ is well-estimated), the error and burn-in become solely that of regressing $F \in \calF$ on each task.
Now, combining \Cref{prop: task coverage -> task diversity}, \Cref{prop: iid NRLS bound}, and \Cref{prop: iid estimation error bound} yields the final bound on the excess transfer risk.
\begin{restatable}[Transfer risk bound]{theorem}{iidMainBound}\label{thm: iid main bound}

    Let \Cref{ass: function class boundedness} and \Cref{ass: hypercontractivity} hold. Assume $\sfP^{0:T}$ satisfy $\mu_{\sfX}$-\eqref{eq: mu task coverage}, and let $\mu_F$ be defined as in \eqref{eq: head coverage}. Define $C_Z$, $\hhatZ$ and $\hhatV$ as in \Cref{prop: iid NRLS bound}. With probability at least $1 - \delta$, the target excess risk of the two-stage ERM \eqref{eq: two-stage ERM} predictor $(\Fzerohat, \hat g)$ is bounded by
    \ifshort{
    \begin{align*}
        &\ER(\Fzerohat, \hat g) \leq \frac{\sigma_{\sfW}^2 C_Z \dy r \log(1/\delta)}{N'}\\
        &+ \sigma_{\sfW}^2 \mu_{\sfX} \snorm{\bfF_\star^{(0)}(\bfF^{1:T}_\star)^\dagger}_2 \cdot \tilde\calO\paren{\frac{\dy r}{N} + \frac{\log\abs{\calG} + \log(1/\delta)}{NT}},
    \end{align*}
    }
    {
    \begin{align*}
        \ER^{(0)}(\Fzerohat, \hat g) &\leq \frac{\sigma_{\sfW}^2 C_Z \dy r \log(1/\delta)}{N'}\\
        &\quad + \mu_{\sfX} \mu_F \sigma_{\sfW}^2 \Bigg(\frac{\dy r}{N}\log\paren{\frac{B_\calF B_\calG NT}{\sigma_\sfW} } + \blue{\frac{\log\abs{\calG} + \log(1/\delta)}{NT}}  \Bigg)
    \end{align*}
    }
    as long as the following burn-in conditions hold:
    \begin{align*}
        N' &\gtrsim  C_Z \sqrt{C^{(0)}_{4 \to 2}} r + \hhatZ^2 \log(1/\delta),\quad  N' \gtrsim \hhatV^2 \paren{\frac{\log(1/\delta)}{\log(N')}}^{8}\\
        N &\gtrsim C^{1:T}_{4 \to 2} \paren{ \dy r \log\paren{\frac{B_\calF B_\calG NT}{\sigma_\sfW}} +  \blue{\frac{\log\abs{\calG} + \log(1/\delta)}{T}}}.
    \end{align*}
\end{restatable}
The proofs of \Cref{prop: iid estimation error bound} and \Cref{thm: iid main bound} can be found in \Cref{appdx: estimation error}. We observe the following: 1.\ the rates are qualitatively correct, where the noise-level hits $\mathrm{dim}(\calF)/\scurly{N, N'}$ for the complexity of fitting the linear heads and $\log \abs{\calG}$ for the shared representation, 2.\ the burn-in for $N'$ is proportional to $r$, which is the number of samples necessary for $\Fzerohat$ to be well-posed, 3.\ in the burn-in for $N$, $\log\abs{\calG}$ is additionally divided by $T$. Therefore, for large $T$, the dominant term is $\dy r$. Compared to prior work in nonlinear representation learning \citep{du2020few, tripuraneni2020theory} (where $\dy = 1$), this is a dramatic improvement from at least $\calO(\dx)$ to $r$.


\subsection{Representation Learning with Little Mixing}\label{sec: main; rep learning with little mixing}

In this section, we extend our results to full generality, allowing possibly dependent within-task data within-task and general representation classes $\calG$, subsuming various settings of interest, such as identification of nonlinear dynamical systems. Beyond finiteness, we demonstrate that the statistical complexity of a representation class can be bounded via its log-covering number $\log \sfN_\infty(\calG, \gamma)$ in the supremum metric $\rho(g_1, g_2) = \sup_{x \in \sfX} \norm{g_1(x) - g_2(x)}_2$. In particular, this allows a painless instantiation of various standard classes of interest, such as (Lipschitz) parametric function classes.
\begin{definition}[Lipschitz parametric function class]\label{def: Lipschitz loss class}
    A function class $\calG$
    is called $(B_\theta, L_\theta, d_\theta)$-\emph{Lipschitz parametric}
    if $\calG = \{ g_\theta(\cdot) \mid \theta \in \Theta \}$ with $\Theta \subset \R^{d_\theta}$, and satisfies
    \begin{align}\label{eq: Lipschitz loss class}
        &\sup_{\theta \in \Theta} \norm{\theta} \leq B_\theta, \\
        &\sup_{x \in \sfX} \sup_{\substack{\theta_1, \theta_2 \in \Theta \\ \theta_1 \neq \theta_2}} \frac{\norm{g_{\theta_1}(x) - g_{\theta_2}(x)}_2}{\norm{\theta_1-\theta_2}_2} \leq L_\theta.
    \end{align}
    By a standard volumetric argument \citep{wainwright2019high}, it can be shown that a $(B_\theta, L_\theta, d_\theta)$-Lipschitz parametric class $\calG$ satisfies
    \[
    \log \sfN_\infty(\calG, \gamma) \leq d_\theta \log\paren{1 + \frac{2B_\theta L_\theta}{\gamma}}.
    \]
\end{definition}
Parametric function classes include various models of interest, such as (generalized) linear models and neural networks with smooth activations. Notably, instantiating $\calG$ as a linear class, by identifying it with $r \times \dx$ (orthonormal) matrices \citep{du2020few}, we may replace $\log \abs{\calG} \mapsto \tilde \calO(r \dx)$, immediately recovering the rates from prior work on multi-task linear regression, along with the reduced burn-in and refined task diversity estimate. We note that our results are not limited to ``parametric-type'' covering number estimates, and can handle various non-parametric classes. We refer to \citep{ziemann2022learning} for various worked examples; the resulting effect on the martingale complexity bound and thus the final risk bound is elucidated in \Cref{lem: multi-task martingale complexity}. In particular, by associating the complexity of $\calG$ to a well-studied measure in the log-covering number, rate-optimal multi-task bounds can be easily extended from many existing single-task settings, avoiding the need for custom complexity measures that may be hard to instantiate or suboptimal. To quantify dependency, we define $\phi$-mixing \citep{kuznetsov2017generalization} covariates.
\begin{definition}[$\phi$-mixing]
    A sequence of random variables $\{S_i\}_{i=1}^n$ is called $\phi$-mixing if
    \begin{align*}
        \phi(i) \triangleq \sup_{t\in[n]: t+i\leq n} \sup_{s} \| \sfP_{S_{i+t}}( s \mid S_{1:t} ) - \sfP_{S_{i+t}} \|_{\mathsf{TV}} < \infty,\quad  i \in [n],
    \end{align*}
    where $\norm{\cdot}_{\mathsf{TV}}$ denotes the total variation distance.
\end{definition}
In other words, $\phi$-mixing measures the distributional distance between the marginal distribution of a covariate at a given time and the same distribution \textit{conditioned} on events in the past. The mixing function $\phi$ measures how the dependency decays as a function of timesteps in the past $i$. Therefore, a standard technique applied to mixing processes is \emph{blocking}; that is, chopping up trajectories into contiguous blocks of samples. By $\phi$-mixing, given sufficiently long blocks covariates from one block are essentially independent from covariates from more than one block away. We give a short preliminary on the blocking technique in \Cref{appdx: blocking}. Unfortunately, standard applications of the blocking technique deflate the iid rate by a factor of block length, which is undesirable, especially for slowly mixing processes. However, using recent insights from \cite{ziemann2023noise} and \cite{ziemann2022learning}, we observe that the effect of mixing can be relegated to \emph{solely affecting the burn-in}. In other words, we demonstrate that past a mixing-inflated burn-in, the risk bounds remain the same from the earlier independent-samples results.
With this preliminary in place we now state the analogue of \Cref{prop: iid NRLS bound}.

\begin{restatable}{proposition}{mixNRLSbound}\label{prop: mix NRLS bound}
Suppose that $\Pzero$ is stationary and $\phi$-mixing and fix $\delta \in (0,1)$. Fix a block length $k$ dividing $N'/2$. Define the \emph{blocked} noise-class interaction term $V \triangleq \frac{1}{k} \sum_{i=1}^k U_iZ_i^\top \mathopen{}\Sigmaz_Z\mathclose{}^{-1/2}$ and $\sigma_V^2 \triangleq \Ezero\brac{\snorm{V}_F^2}$. Define $\sigma_U^2$, $C_{Z}$, $\hhatZ$ and $\hhatV$ as in \Cref{prop: iid NRLS bound}. As long as the burn-in conditions hold:
\begin{align*}
        \frac{N'}{k} \gtrsim r + \hhatZ^2 \log(1/\delta), \quad \frac{N'}{k} \phi(k) \leq \delta,  \quad \frac{N'}{k} \gtrsim \hhatV^2 \paren{\frac{\log(1/\delta)}{\log(N')}}^{8},
    \end{align*}
then with probability at least $1 - \delta$ we have
    \ifshort{
    \begin{align*}
        &\snorm{(\Fzerohat - \Fzerohatstar)\sqrt{\Sigmaz_Z}}_F^2 \\
        &\lesssim \frac{\sigma_V^2 \log(1/\delta)}{N'} 
        \lesssim \frac{C_Z \sigma_U^2 r \log(1/\delta)}{N'} +\\
        &\frac{C^{(0)} k\nu^{-1} \log (1/\delta) }{N'T}\sum_{t=1}^T \Et \snorm{\Fthat \hat g(X) - \Ftstar g_\star (X)}_2^2,
    \end{align*}
    }
    {
    \begin{align*}
        \snorm{(\Fzerohat - \Fzerohatstar)\sqrt{\Sigmaz_Z}}_F^2
        &\lesssim \frac{\sigma_V^2 \log(1/\delta)}{N'} \\
        &\lesssim \frac{C_Z \sigma_U^2 r \log(1/\delta)}{N'}.
    \end{align*}
    }
\end{restatable}
The proof is analogous to \Cref{prop: iid NRLS bound}, requiring some additional analysis to bound the \emph{blocked} noise-class variance term $\sigma_V^2$. However, we see that the final bound remains identical to \Cref{prop: iid NRLS bound}, unaffected by the block-length $k$. In other words, we recover \Cref{prop: iid NRLS bound} and are able to shift the effect of mixing to the burn-in. Additionally, as we noted earlier, a key benefit of the martingale offset complexity is that it does not depend on the data distribution beyond the conditional noise-level.
Therefore, with minimal modifications to the task-averaged estimation error bound (\Cref{prop: iid estimation error bound}) besides substituting a general parametric $\calG$ in lieu of a finite $\calG$, we yield our main theorem.
\begin{restatable}[Transfer risk bound, mixing]{theorem}{mixMainBound}\label{thm: main bound mixing}
    Let \Cref{ass: function class boundedness} and \Cref{ass: hypercontractivity} hold. Assume $\sfP^{0:T}$ satisfy $\mu_{\sfX}$-\eqref{eq: mu task coverage}, and let $\mu_F$ be defined as in \eqref{eq: head coverage}.
    Suppose that $P^{0:T}$ are each stationary and $\phi$-mixing. Assume that $k$ is fixed and divides $N'/2$ and $N/2$. Define the quantity $\Phi \triangleq (\sum_{i=1}^\infty \sqrt{\phi(i)})^2$. Assume $\calG$ admits a $(B_\theta, L_\theta, d_\theta)$-Lipschitz parametric form (\Cref{def: Lipschitz loss class}).
    With probability at least $1 - \delta$, the target excess risk of the two-stage ERM \eqref{eq: two-stage ERM} predictor $(\Fzerohat, \hat g)$ is bounded by
    \ifshort{
    \begin{align*}
        &\ER^{(0)}(\Fzerohat, \hat g) \leq \frac{\sigma_{\sfW}^2 C_Z \dy r \log(1/\delta)}{N'}\\
        &+ \sigma_{\sfW}^2 \mu_{\sfX} \snorm{\bfF_\star^{(0)}(\bfF^{1:T}_\star)^\dagger}_2 \cdot \tilde\calO\paren{\frac{\dy r}{N} + \frac{\mathrm{C}(\calG) + \log(1/\delta)}{NT}},
    \end{align*}
    }
    {
    \begin{align*}
        \ER^{(0)}(\Fzerohat, \hat g) &\lesssim \frac{\sigma_{\sfW}^2 C_Z \dy r \log(1/\delta)}{N'}\\
        &\quad + \mu_{\sfX} \mu_F \sigma_{\sfW}^2 \Bigg(\frac{\dy r}{N}\log\paren{\frac{B_\calF B_\calG NT}{\sigma_\sfW} } + \blue{\frac{d_\theta \log\paren{\frac{B_\calF B_\theta L_\theta NT}{\sigma_\sfW}} + \log(1/\delta)}{NT}} \Bigg),
    \end{align*}
    }
    as long as the following burn-in conditions hold:
    \begin{align*}
        \frac{N'}{k} &\gtrsim C_Z \sqrt{C^{(0)}_{4 \to 2}} r + \hhatZ^2 \log(1/\delta), \quad \frac{N'}{k} \phi(k) \leq \delta,  \quad \frac{N'}{k} \gtrsim \hhatV^2 \paren{\frac{\log(1/\delta)}{\log(N')}}^{8} \\
        \frac{N}{\Phi} &\gtrsim C^{1:T}_{4 \to 2} \paren{ \dy r \log\paren{\frac{B_\calF B_\calG NT}{\sigma_\sfW}} +  \blue{\frac{d_\theta \log\paren{\frac{B_\calF B_\theta L_\theta NT}{\sigma_\sfW}} + \log(1/\delta)}{T}}}.
    \end{align*}
\end{restatable}
To understand how mixing affects our bounds, let us consider geometric $\phi$-mixing, i.e.\ $\phi(k) \leq \Gamma \rho^k$ for some $\Gamma>0$, $\rho \in (0,1)$. Then we can find a valid block length $k = \frac{\log(\Gamma N'/\delta)}{\log(1/\rho)}$ and $\Phi \leq \frac{\Gamma}{(1 - \sqrt{\rho})^2}$, thereby inflating the burn-in requirement on $N'$ by a factor of $\approx$ $\log(N'/\delta)$ and $N$ by a constant factor. Notably, the excess risk bound remainds unchanged between \Cref{thm: iid main bound} and \Cref{thm: main bound mixing} (up to universal constants). We also note that the problem-specific constants $C_Z, C^{(0)}_{4 \to 2}, C^{(1:T)}_{4 \to 2}$ are defined over expectations of respective \emph{mixture} distributions. Therefore, by linearity of expectation these constants remain the same between sampling dependent trajectories versus sampling each datapoint independently from its marginal distribution, thus remain the same from the iid setting. With $\phi$-mixing, we are able to port to broader sequential settings, such as Markov Chains \citep{samson2000concentration} and parametrized dynamical systems \citep{tu2022learning, ziemann2022learning}.

%% file: icml/icml_tex/numerical.tex
\section{Numerical Validation}\label{sec: numerical}

\begin{figure}
    \centering
    \subfigure[Example observation]{
    \label{fig: example obs} \includegraphics[width=0.3\linewidth]{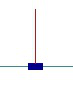}}
    \subfigure[ideal keypoint extraction]{
    \label{fig: ideal keypoints}
    \includegraphics[width=0.5\linewidth]{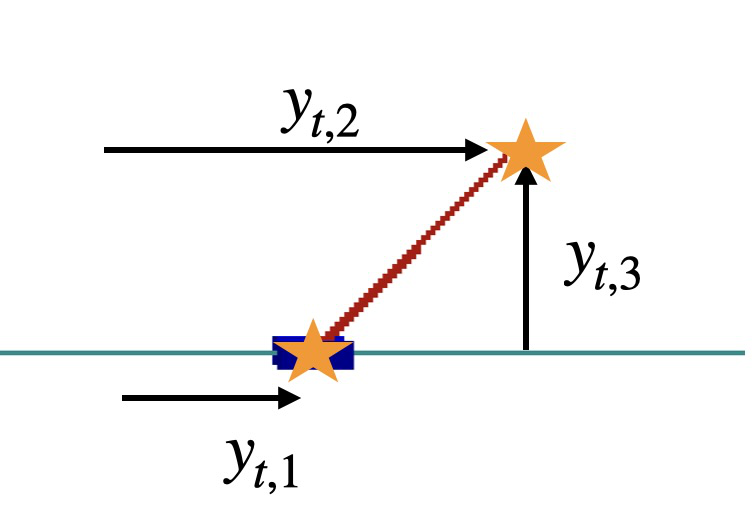}}
    \caption{\Cref{fig: example obs} shows an example camera observation of the pybullet simulated cartpole environment. In this image, the cartpole is at the state $x = \bmat{0 & 0 & 0 & 0}^\top$. \Cref{fig: ideal keypoints} illustrates the ideal keypoints extracted from a cartpole image.}
\end{figure}

\begin{figure*}[t]
    \centering
    \subfigure[Input imitation error]{
        \includegraphics[width=0.26\textwidth]{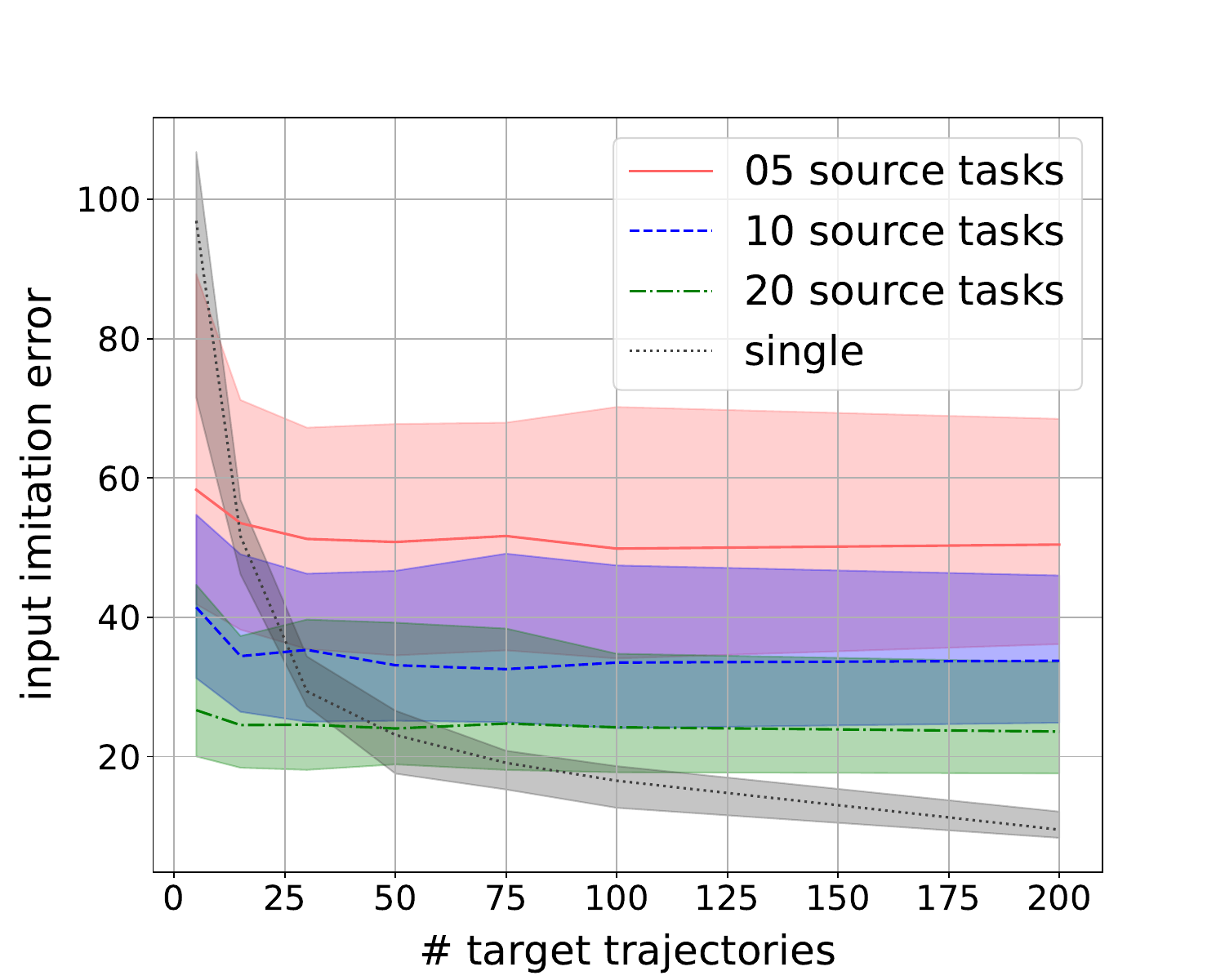}
    }
    \subfigure[State imitation gap]{
    \includegraphics[width=0.26\textwidth]{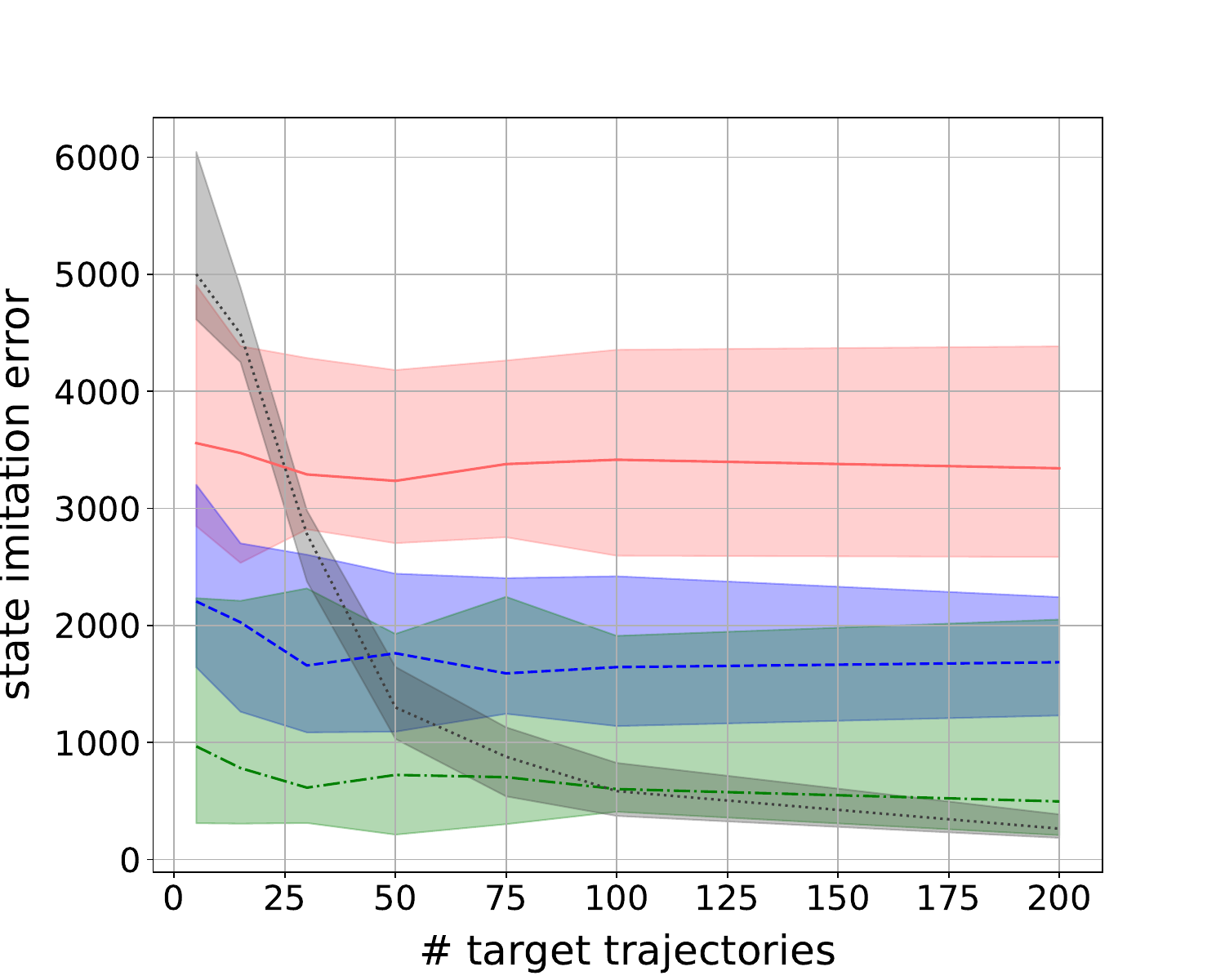}}
    \subfigure[Percent stable]{
    \includegraphics[width=0.26\textwidth]{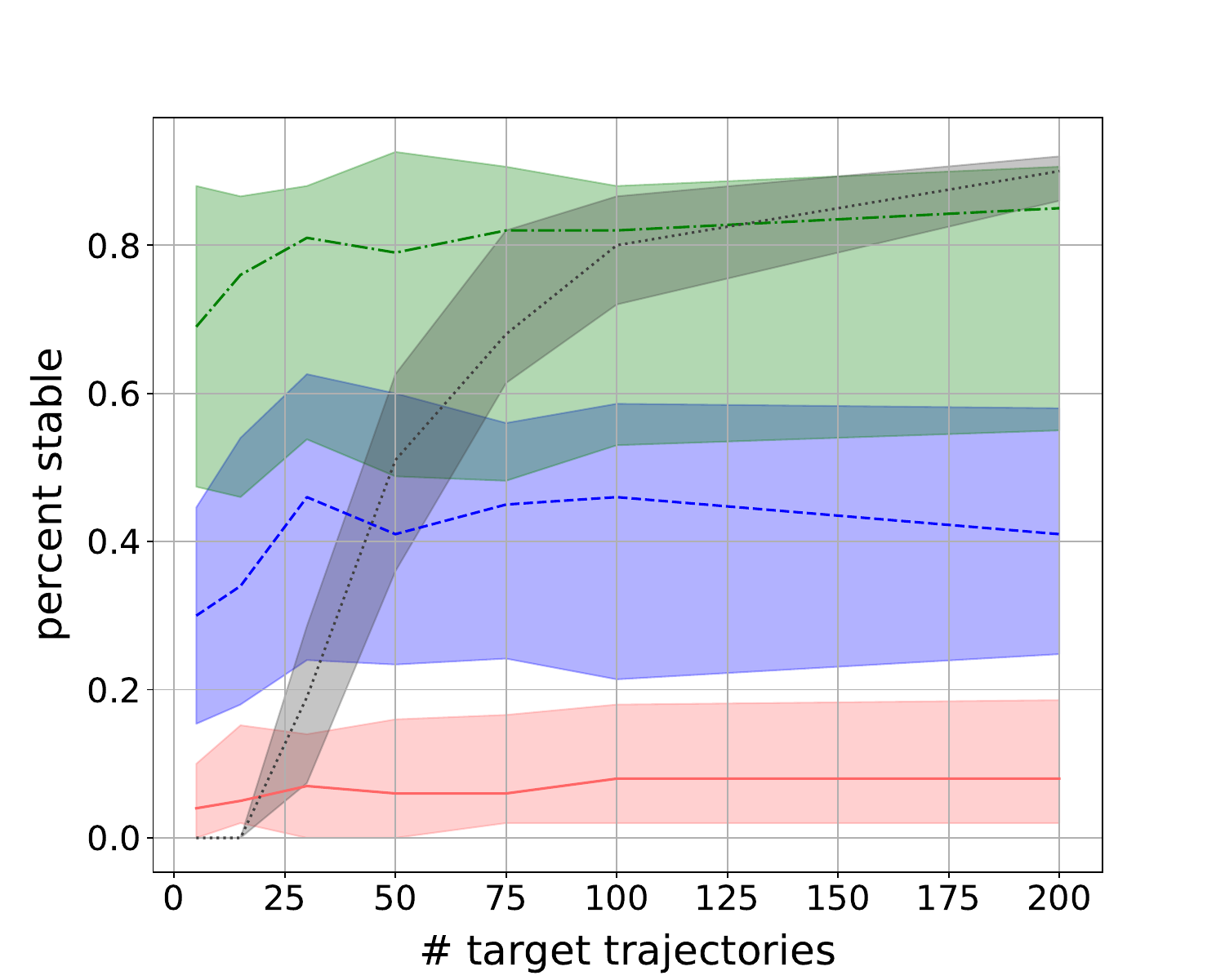}}
    \caption{Three evaluation metrics comparing the performance of multi-task versus single-task imitation learning: the MSE between the input of the learned and expert controllers when evaluated on the expert trajectory, the deviation between the state trajectories generated by the learned and expert controllers, and the \%trials that the learned controller keeps the pole balanced for all $500$ timesteps (the dynamics are discretized to $\Delta t = 0.02$ seconds). Three curves are shown for multi-task imitation learning, generated by pre-training with a different number of source tasks. In all metrics, multi-task learning improves over single task when few target trajectories are available.}
    \label{fig: cartpole imitation learning comparisons}
    \vspace{-0.3cm}
\end{figure*}

To validate our theoretical observations, we consider a non-trivial regression task over dynamical systems: balancing a pole atop a cart from visual observations, as pictured in \Cref{fig: example obs}. 
A collection of systems is obtained by randomly sampling different values for the cart mass, pole mass, and pole length parameters. The regression task is to imitate expert policies controlling each collection of systems from (control input, observation) pairs. We design expert policies as linear controllers of the underlying state\footnote{This consists of the cart position and velocity, and pole angular position and angular velocity.} to balance the pole in the upright position. The expert estimates the state of the system from the camera observations by first applying a keypoint extractor to the camera observations to get noisy estimates of two keypoints (visualized in \Cref{fig: ideal keypoints}), and then passing these noisy estimates into a Kalman filter. A common keypoint extractor is shared across the experts, but the linear controllers and filters are system-specific. Actuation noise is added to the expert input when it is applied to the system. We use demonstrations from the aforementioned expert policies to train imitation learning policies to replicate the experts. The policies are parameterized with convolutional neural networks. They take as input a history of $8$ images, and output the control action to be applied to the system. The policies are trained by solving a supervised learning problem using the expert demonstrations.

Our theoretical analysis predicts that multi-task learning helps substantially in this setting, due to the shared keypoint extractor across all policies. The part that varies between expert policies is the controller and filter, which are linear maps from the keypoints to the control action to be consistent with our linear $\calF$ nonlinear $\calG$ model.  

The experimental results in \Cref{fig: cartpole imitation learning comparisons} compare multi-task learning with single-task learning. We consider multi-task learning using a varying number of source tasks, each consisting of $10$ expert demonstrations. The $x$-axis denotes the number of demonstrations available from the target task. For single task learning, these trajectories are used to train the entire network, while for multi-task learning they are used to fit only the final layer, keeping the representation fixed from pre-training on the source tasks. Three evaluation metrics are plotted: the MSE of the learned controller inputs, the MSE between the learned and expert trajectories, and the \%trials where the controller is stabilizing. Each metric is averaged over $50$ evaluation rollouts for each controller.  We plot the median and shade 30\%-70\% quantiles for these evaluation metrics over 5 random seeds for pretraining the representation, and 10 realizations of target tasks. In all metrics, multi-task learning improves over single task learning in the low data regime as predicted, but saturates quickly when the number of target trajectories exceeds the number of per-task training trajectories, which our theory predicts is the limiting rate when $T$ is large. Full experimental details are contained in \Cref{appdx: numerical}.

%% file: icml/icml_tex/discussion_conclusion.tex
\section{Discussion}\label{sec: disc + conclusion}

We provided new guarantees for nonlinear representation learning that: 1.\ agree with prior work rate-wise, 2.\ apply to non-identical covariates and/or sequentially dependent ($\phi$-mixing) covariates, 3.\ improve the per-task sample requirement and refine the task-diversity measure.
We did not address pathologies that can arise in multi-task learning, such as class (source data) imbalance and low task diversity. Indeed, addressing these pathologies is what motivates ongoing work in active learning \cite{wang2023improved} and alignment \citep{wu2019understanding}, which are important directions to fully realize the benefit of learning over multiple tasks. 


%% file: icml/icml_tex/appdx_theory_results.tex
\section{Proofs and Additional Information for \Cref{sec: main results}}\label{appendix: main results}
\sloppy
\subsection{Proofs for \Cref{sec: canonical decomposition}}\label{appdx: canonical decomposition}

\TaskCoverageTaskDiversity*

\begin{proof}
    Let $g$ be fixed. Then, writing out the left-hand side of \eqref{eq: nu task diversity}, we have:
\begin{equation}\label{eq:lin_id_1}
\begin{aligned}
        &\hspace{-30pt}\inf_{F\in \calF}\Ezero \|Fg(X)- Y\|_2^2- \Ezero \|\Fzerostar g_\star(X)-Y\|_2^2\\
        & =\inf_{F \in \calF}\Ezero \|F g(X)-\Fzerostar g_\star(X)\|_2^2 && (\Fzerostar\textnormal{ is } \Ezero-\textnormal{optimal}).
\end{aligned}
\end{equation}
We make repeated use of the following fact: for any $t=0,\dots,T$:
\begin{align}
\begin{split}\label{eq: partial quadratic form}
    \inf_{F \in \calF}\Ezero \|F g(X)-\Ftstar g_\star(X)\|_2^2 &= \inf_{F \in \calF} \trace \paren{ \begin{bmatrix} F^\top \\ -\Ftstar^\top \end{bmatrix}^\top \Sigma^{(t)}_g \begin{bmatrix} F^\top \\ -\Ftstar^\top \end{bmatrix}} \\
    &= \trace\paren{\Ftstar \overline\Sigma_g^{(t)} \Ftstar^\top },
\end{split}
\end{align}
where $\Sigma^{(t)}_g, \overline\Sigma_g^{(t)}$ are defined in \Cref{def: mu task coverage}, and the optimization step is a standard calculation about partial minima of quadratic forms \citep[see e.g.][Example 3.15, Appendix A.5.4]{boyd2004convex}. Notably, the result therein is defined for vector arguments; however, this extends to matrix arguments straightforwardly by treating each column of $\begin{bmatrix} F^\top \\ -\Ftstar^\top \end{bmatrix}$ individually. Applying \eqref{eq: partial quadratic form} to task $0$, we have
\begin{align*}
    \inf_{F\in \calF} \ER^{(0)}(F, g) & = \trace\paren{\Fzerostar \overline\Sigma_g^{(0)} \Fzerostar^\top }.
\end{align*}
Meanwhile, applying \eqref{eq: partial quadratic form} to the RHS of \eqref{eq: nu task diversity} yields:
\begin{align*}
    \frac{\nu^{-1}}{ T }\sumT\inf_{F^{(t)}}\; \Et \|F^{(t)}  g(X)- \Ftstar g_\star(X)\|_2^2 &= \frac{\nu^{-1}}{ T }\sumT \trace\paren{\Ftstar \overline\Sigma_g^{(t)} \Ftstar^\top }.
\end{align*}
To provide a valid bound on $\nu^{-1}$, we do the following:
\begin{align*}
    \inf_{F\in \calF} \ER^{(0)}(F, g) &= \trace\paren{\Fzerostar \overline\Sigma_g^{(0)} \Fzerostar^\top } \\
    &= \trace\paren{\Fzerostar (\bfF^{1:T})^{\dagger/2} (\bfF^{1:T})^{1/2} \overline\Sigma_g^{(0)} (\bfF^{1:T})^{1/2} (\bfF^{1:T})^{\dagger/2} \Fzerostar^\top } \tag{requires $\range(\bfF_\star^{(0)}) \subseteq \range(\bfF^{1:T}_\star)$} \\
    &\leq \mu_F \trace\paren{\overline\Sigma_g^{(0)} \bfF^{1:T}} \\
    &= \frac{\mu_F }{T} \sumT \trace\paren{\Ftstar \overline\Sigma_g^{(0)} \Ftstar^\top} \\
    &\leq  \frac{\mu_X \mu_F}{T} \sumT \trace\paren{\Ftstar \overline\Sigma_g^{(0)} \Ftstar^\top} \tag{$\overline\Sigma_g^{(0)} \preceq \mu_\sfX \overline\Sigma_g^{(t)}$ by $\mu_\sfX$-\eqref{eq: mu task coverage}} \\
    &= \frac{\mu_X \mu_F}{T} \sumT\inf_{F^{(t)}}\; \Et \|F^{(t)}  g(X)- \Ftstar g_\star(X)\|_2^2.
\end{align*}
This completes the proof.

\end{proof}

\subsection{Directly Estimating $\nu$}\label{appdx: est task diversity}

\taskdivconsistency*

\begin{proof}
    We recall the estimator
    \begin{align*}
    \hat\nu_N(g) \triangleq \frac{\frac{1}{T}\sumT \hatEt[\snorm{Y}_2^2] - \trace\paren{\widehat F_{g}^{(t)} \hatEt[g(X)g(X)^\top] \mathopen{}\widehat F_{g}^{(t)}\mathclose{}^\top}}{ \hatEzero \brac{\snorm{Y}_2^2} - \trace\paren{\widehat F_{g}^{(0)} \hatEzero[g(X)g(X)^\top] \mathopen{}\widehat F_{g}^{(0)}\mathclose{}^\top} }.
    \end{align*}
    Toward establishing the consistency of $\hat\nu_N(g)$, it suffices to demonstrate for each $t = 0,\dots, T$,
    \begin{align*}
         \hatEt[\snorm{Y}_2^2] - \trace\paren{\widehat F_{g}^{(t)} \hatEt[g(X)g(X)^\top] \mathopen{}\widehat F_{g}^{(t)}\mathclose{}^\top}
    \end{align*}
    is a consistent estimator of $\inf_F \ER^{(t)}(F,g)$. From \eqref{eq: partial quadratic form}, we also have that
    \begin{align*}
        \inf_F \ER^{(t)}(F,g) = \trace\paren{\Ftstar \overline \Sigma_g^{(t)} \Ftstar^\top}.
    \end{align*}
    Expanding out $\overline\Sigma_g^{(t)}$ (see \Cref{def: mu task coverage}), we have
    \begin{align*}
        &\trace\paren{\Ftstar \overline \Sigma_g^{(t)} \Ftstar^\top} \\
        =\;& \trace(\Ftstar \Et[g_\star(X)g_\star(X)^\top] \Ftstar^\top) - \trace(\Ftstar \Et[g_\star(X)g(X)^\top] \; \Et[g(X) g(X)^\top]^{\dagger} \;\Et[g(X)g_\star(X)^\top] \Ftstar^\top).
    \end{align*}
    Observing that
    \begin{align*}
        \yit &= \Ftstar g_\star(\xit) \\
        \widehat F^{(t)}_g &\triangleq \argmin_{F} \;\hatEt[\snorm{Y - F g(X)}_2^2] \\
        &= \hatEt[Y g(X)^\top]\; \hatEt[g(X)g(X)]^{\dagger} \\
        &= \Ftstar \hatEt[g_\star(X)g(X)^\top]\; \hatEt[g(X)g(X)]^{\dagger},
    \end{align*}
    we compute the population counterparts of $\snorm{Y}_F^2$ and $\snorm{\widehat F_{g}^{(t)} g(X)^\top}_F^2$:
    \begin{align*}
        &\Et[\snorm{Y}_2^2] \\
        &\quad = \trace\paren{\Ftstar \Et[g_\star(X) g_\star(X)^\top] \Ftstar^\top} \\
        &\Et\trace\paren{\widehat F_{g}^{(t)} \hatEt[g(X) g(X)^\top] \mathopen{}\widehat F_{g}^{(t)}\mathclose{}^\top} \\
        &\quad = \trace\paren{\Ftstar \Et\brac{\hatEt[g_\star(X)g(X)^\top] \; \hatEt[g(X) g(X)^\top]^{\dagger} \;\hatEt[g(X)g_\star(X)^\top]} \Ftstar^\top}
    \end{align*}
    Notably, the first term on the RHS of the least-squares expression converges to
    \begin{align*}
        \trace(\Ftstar \Et[g_\star(X)g(X)^\top] \; \Et[g(X) g(X)^\top]^{\dagger} \;\Et[g(X)g_\star(X)^\top] \Ftstar^\top)
    \end{align*}
    via the law of large numbers. Putting this together, this verifies
    \begin{align*}
        &\hatEt[\snorm{Y}_2^2] - \trace\paren{\widehat F_{g}^{(t)} \hatEt[g(X)g(X)^\top] \mathopen{}\widehat F_{g}^{(t)}\mathclose{}^\top}  \\
        &\toP \trace(\Ftstar \Et[g_\star(X)g_\star(X)^\top] \Ftstar^\top) - \trace(\Ftstar \Et[g_\star(X)g(X)^\top] \; \Et[g(X) g(X)^\top]^{\dagger} \;\Et[g(X)g_\star(X)^\top] \Ftstar^\top) \\
        &\quad=\trace\paren{\Ftstar \overline \Sigma_g^{(t)} \Ftstar^\top}  \\
        &\quad= \inf_F \ER^{(t)}(F, g).
    \end{align*}
    Therefore, $\hat\nu(g) \toP \nu(g)$.
    
\end{proof}

\subsection{Non-realizable Least Squares}\label{appdx: non-realizable least squares}

 \iidNRLSbound*
\begin{proof}
We defer the derivation of the burn-in requirements to after \Cref{lem: ziemann LS result redux} (noting in the iid setting $k = 1$). The result in terms of $\sigma_V$ is immediate by \citet[Theorem 3.1]{ziemann2023noise}. We postpone discussion of adapting the result to the full dependent covariate case in the proof of \Cref{prop: mix NRLS bound}. It remains to compute the noise term $\sigma_V$ for the second inequality. Namely, we have that:
\begin{equation}\label{eq: sigmav to sigmau}
    \begin{aligned}
        \sigma_V^2 &=\sfE\brac{ \|UZ^\top \Sigma^{-1/2}_Z\|^2_F}\\
        &= \sfE \brac{\|U\|_2^2 \|\Sigma^{-1/2}_Z Z\|^2_2} && (\textnormal{rewrite rank-1 objects})
        \\
        &\leq \sqrt{\sfE \|U\|_2^4  \sfE \|\Sigma^{-1/2}_Z Z \|^4_2}. &&(\textnormal{Cauchy-Schwarz})
    \end{aligned}
\end{equation}
The result follows since $\sfE \|\Sigma^{-1/2}_Z Z \|^4_2 \leq C_Z^2 r^2$.
\end{proof}

Let us spend a few moments to establish the existence of $\hhatZ, \hhatV$. First off, we assume without harm that $\Sigmaz_Z$ is invertible: $\hhatZ$ by definition considers only $v$ such that $v^\top \Sigmaz_Z v =1$ and is thus agnostic to rank-degeneracy. Notably, by using the definition of subgaussianity we immediately get that $\tilde Z \triangleq (\Sigmaz_Z)^{-1/2}Z$ is also subgaussian with corresponding variance proxy no worse than $B_\calG^2 / \lambda_{\min}^{+}(\Sigmaz_Z)$, where $\lambda_{\min}^{+}$ denotes the smallest non-zero eigenvalue. This implies we may bound $\hhatZ$ by
\begin{align*}
    \hhatZ^2 &\triangleq \max_{v^\top \Sigmaz_Z v = 1} \Ezero[\ip{v, Z}^4] \\
    &= \max_{\norm{u}=1} \Ezero[\ip{u, \tilde Z}^4] \\
    &\lesssim (4^{1/2} \mathrm{subG}(\tilde Z))^4 \\
    &\lesssim (B_\calG^2 / \lambda_{\min}^{+}(\Sigmaz_Z))^4,
\end{align*}
where we used the fact that $\tilde Z$ is a subgaussian random vector with parameter no larger than $B_\calG^2 / \lambda_{\min}^{+}(\Sigmaz_Z)$, and thus for any $\norm{u}=1$, $\ip{u, \tilde Z}$ is a subgaussian random variable with the said variance proxy. As for $\hhatV$, by definition $V = UZ^\top (\Sigmaz_Z)^{-1/2}$, where $U = Y - \Fzerohatstar Z = \Fzerostar g_\star(X) - \Fzerohatstar Z$. Therefore, as a sum of subgaussian vectors $U$ is subgaussian; from the prior discussion $\tilde Z$ is subgaussian, and thus $V$ as a product of subgaussian vectors is thus at worst subexponential. Therefore, we know $\norm{\snorm{V}_F}_{\Psi_1}$ exists.

\noiselevellemma*

\begin{proof}
Recall that $U = Y - \Fzerohatstar \hat g(X)=W + \Fzerostar g_\star (X)- \Fzerohatstar \hat g(X)$.   By orthogonality (in $L^2$) of $W_i$ to $\Fzerohatstar \hat g(X)$ thus have that:
    \begin{equation} 
        \begin{aligned}
            \sigma^2_U &= \sqrt{\sfE \|U\|^4} 
            \\
            &
            =
            \sqrt{\sfE \|W\|^4 +\sfE\|\Fzerostar g_\star (X)- \Fzerohatstar \hat g(X)\|^4  } 
            \\
            &
            \leq 
            \sqrt{\sfE \|W\|^4} +\sqrt{\sfE\|\Fzerostar g_\star (X)- \Fzerohatstar \hat g(X)\|^4}
            && (\textnormal{Triangle inequality})
            \\
            &
            \lesssim 
            \dy \sigma^2_W +\sqrt{C^{(0)}_{4\to 2}}\sfE\|\Fzerostar g_\star (X)- \Fzerohatstar \hat g(X)\|^2
            &&(\textnormal{hypercontractive estimate and sub-Gaussianity})
            \\
            &
            \lesssim 
             \dy\sigma^2_W +\frac{\sqrt{C^{(0)}_{4\to 2}}\nu^{-1}}{T} \sum_{t=1}^T \Et \snorm{\Fthat \hat g(X) - \Ftstar g_\star (X)}_2^2 
            &&(\textnormal{task-diversity assumption, \Cref{ass: id}})
        \end{aligned}
    \end{equation}
\end{proof}

\mixNRLSbound*

\begin{proof}
    As noted the argument is identical to that presented in \Cref{prop: iid NRLS bound}. By assumption that the (equal) block length $k$ divides $N'$ and the process $\scurly{x^{(0)}_i, y^{(0)}_i}$ is assumed stationary; therefore we may define w.l.o.g.\ the blocked noise-class interaction term $V \triangleq \frac{1}{k} \sum_{i=1}^k U_iZ_i^\top \mathopen{}\Sigmaz_Z\mathclose{}^{-1/2}$ on the first $k$ indices. Therefore, we may make a few adaptations to \citet[Theorem 3.1]{ziemann2023noise} to yield the following (misspecified) least-squares bound on the intermediate covariates $Z \triangleq \hat g(X)$:
    \begin{lemma}[Adapted version of {\citet[Theorem 3.1]{ziemann2023noise}}]\label{lem: ziemann LS result redux}
        Assume $k$ divides $N'$ and there exists a constant $\mathrm h$ such that for all $v:\;v^\top \Sigmaz_Z v = 1$, $\sfE[\ip{v, Z}^4] \leq \mathrm h^2 \sfE[\ip{v, Z}^2]$. Then, as long as the following burn-in conditions hold:
        \begin{align*}
            &\frac{N'}{k} \gtrsim r + \mathrm h^2 \log(1/\delta),\quad  \frac{N'}{k}\phi(k) \leq \delta, \\
            &\paren{\frac{N'}{k}}^{1-2/p} \gtrsim p^2 \frac{\Ezero[\snorm{V}_F^p]^{2/p}}{\Ezero[\snorm{V}_F^2]\delta^{2/p}}\quad \text{for some }p \geq 4,
        \end{align*}
        the following bound holds with probability at least $1 - \delta$:
        \begin{align*}
            \snorm{(\Fzerohat - \Fzerohatstar)\sqrt{\Sigmaz_Z}}_F^2
        &\lesssim \frac{\trace(\sfE[\VEC(V)\VEC(V)^\top]) \log(1/\delta)}{N'} \\
        &=\frac{\sigma_V^2 \log(1/\delta)}{N'}.
        \end{align*}
    \end{lemma}
        
    The conciser form of the above statement compared to \citet[Theorem 3.1]{ziemann2023noise} follows from the fact that we assumed equal block-size and stationarity for convenience, rendering many of the burn-in conditions trivial. Furthermore, rather than use $\trace(\cdot) \equiv \snorm{\cdot}_{2}\; \mathrm{edim}(\cdot)$ on the noise-class variance, for our purposes it suffices to remain with the trace. It remains to analyze the last burn-in. Notably, as written there is a polynomial dependence on $1/\delta$, which is the price paid when $Z$ has finite moments (e.g.\ in \cite{ziemann2023noise} only $4$ moments are assumed), roughly speaking quantifying the transition from the moderate deviations to large deviations regime. However, in our case $Z = \hat g(X)$ is bounded, hence subgaussian--thus we aim to modify this to a Bernstein-type burn-in. Now, invoking our definition of $\hhatV$ from \Cref{def: NRLS moment-equiv quants}, we see that
    \begin{align*}
        \frac{\Ezero[\snorm{V}_F^p]^{2/p}}{\Ezero[\snorm{V}_F^2]} \leq \frac{p^2 \norm{\snorm{V}_F}^2_{\Psi_1}}{\Ezero[\trace(V^\top V)]} = p^2 \hhatV.
    \end{align*}
    Therefore, it suffices to satisfy the more stringent burn-in for some $p \geq 4$:
    \begin{align*}
        \paren{\frac{N'}{k}}^{1-2/p} \gtrsim p^4  \hhatV (1/\delta)^{2/p}.
    \end{align*}
    We now aim to find the optimal range for $p$. Defining $m \triangleq N'/k$, we take log on both sides to yield
    \begin{align*}
        \paren{1 - \frac{2}{p}}\log(m) &\geq 4\log(p) + \log(C\hhatV) + \frac{2}{p}\log(1/\delta), \quad C >0 \text{ is a fixed constant.}
    \end{align*}
    Rearranging and substituting $p \to (C\hhatV)^{-1/4} (N')^{b/4}$, where $b > 0$ is a constant to be determined later, we find the above inequality is equivalent to
    \begin{align*}
        b + 2 \paren{\frac{C\hhatV}{m^b}}^{1/4} \paren{\frac{1 + \log(1/\delta)}{\log(m)}} &\leq 1.
    \end{align*}
    Therefore, sufficing to choose $b = 1/2$ (though we may similarly choose any $b = 1- \varepsilon$), we invert the above inequality to yield the burn-in requirement:
    \begin{align}\label{eq: subG burn-in}
        m = \frac{N'}{k} &\gtrsim \hhatV^2 \paren{\frac{\log(1/\delta)}{\log(N'/k)}}^8.
    \end{align}
    Notably, when $\delta \geq \mathrm{poly}(1/m)$, then the above reduces to $N'/k \gtrsim \hhatV^2$. It now remains to analyze the noise-class variance $\sigma_V^2$. By definition we have that:
    \begin{align*}
        \sigma_V^2 &= \frac{1}{k^2}\trace\left( \Sigma_Z^{-1/2} \sfE \left(\sum_{i,j=1}^k (U_i Z_i)^\top (U_j Z_j)  \right) \Sigma_Z^{-1/2} \right) \\
        &\leq \frac{1}{2k^2} \sfE \sum_{i,j=1}^k \norm{U_i Z_i^\top \Sigma_Z^{-1/2}}_F^2 + \norm{U_j Z_j^\top \Sigma_Z^{-1/2}}_F^2 &&\trace(A^\top B)\leq \frac{\norm{A}_F^2}{2} + \frac{\norm{B}_F^2}{2} \\
        &= \frac{1}{k}\sfE \sum_{i=1}^k \norm{U_i Z_i^\top \Sigma_Z^{-1/2}}_F^2 \\
        &= \sfE \norm{U Z^\top \Sigma_Z^{-1/2}}_F^\top \\
        &\leq \sqrt{\sfE \snorm{U}_2^4 \snorm{\Sigma_Z^{-1/2} Z}_2^4 } \\
        &\leq C_Z r \sigma_U^2. && \text{applying \eqref{eq: sigmav to sigmau}}
    \end{align*}

\end{proof}

\subsection{Bounding the Estimation Error}\label{appdx: estimation error}
The goal is to control the task-averaged estimation error. As previously discussed, the key observation is to quantify a lower isometry, such that
\begin{align*}
    &\frac{1}{T}\sumT \Et \snorm{\ft\mem\circ g(X) - \ftstar\mem \circ g_\star(X)}_2^2 \lesssim \frac{1}{NT}\sumT\sumN \snorm{\ft\mem\circ g(\xit) - \ftstar\mem \circ g_\star(\xit)}_2^2.
\end{align*}
By hypercontractivity, we have an anti-concentration result:
\begin{restatable}[{\citet[Theorem 2]{samson2000concentration}}, {\citet[Prop.\ 5.1]{ziemann2022learning}}]{proposition}{iidLowerIsometry} \label{prop: lower isometry for given fxn}
    Fix $C > 0$. Let $\psi : \sfX \to \R$ be a non-negative function satisfying
    \begin{align}
        \sfE [\psi (X)^2] \leq C \sfE [\psi(X)]^2. \nonumber
    \end{align}
    Then we have:
    \begin{align}
        \sfP\brac{\frac{1}{m} \sum_{i=1}^m \psi(x_i) \leq \frac{1}{2} \sfE [\psi(X)]} \leq \exp\paren{\frac{-m}{8C}}. \nonumber
    \end{align}
\end{restatable}
Setting $\psi(X) \triangleq \snorm{\overline h(X)}_2^2$, \Cref{prop: lower isometry for given fxn} yields a tail bound on the lower-isometry event for a \textit{given} $\overline h$, which we will use shortly.
By an application of the basic inequality \citep{rakhlin2014online, liang2015learning}, an empirical estimation error can be bounded by
\begin{align}
    \frac{1}{NT} \sumT\sumN \snorm{h(\xit)}_2^2
    &\leq \sup_{h \in \calH} \frac{1}{NT} \sumT\sumN 4\ip{\wit, h(\xit)} - \snorm{h(\xit)}_2^2 \\
    &\triangleq \sfM_{NT}(\calH),
\end{align}
where $\sfM_{NT}(\calH)$ is denoted the (empirical) \textit{martingale offset complexity} \citep{liang2015learning, ziemann2022learning}, which serves as the capacity measure of hypothesis class $\calH$. Notably, $\sfM_{NT}(\calH)$ scales with the \textit{noise-level} $\sigma_{\sfW}^2$, rather than the diameter of $\calH$. We control $\sfM_{NT}(\calH)$ via a high-probability chaining bound from \citet[Theorem 4.2.2]{ziemann2022statistical}.
\begin{lemma}[{\citet[Theorem 4.2.2]{ziemann2022statistical}}]\label{lem: high probability chaining}
    Let \Cref{ass: conditional-subG noise} hold, and fix $u, v, w > 0$. Then, with probability at least $1 - 3\exp(-u^2/2) - \exp(-v/2) - e\exp(-w)$, the following bound on the martingale complexity holds:
    \begin{align*}
        \sfM_{NT}(\calH) \leq c \inf_{\gamma > 0, \delta \in [0, \gamma]}\cdot &\bigg\{w\delta \sigma_{\sfW}\sqrt{\dy} + \sqrt{\frac{\sigma_{\sfW}^2}{NT}} \int_{\delta/2}^\gamma \sqrt{\log\sfN_\infty(\calH, \varepsilon)} \;d\varepsilon + \frac{v\sigma_{\sfW}^2}{NT} \\
        &\quad +\frac{\sigma_{\sfW}^2 \log\sfN_\infty(\calH, \gamma)}{NT} + \frac{u\gamma \sigma_{\sfW}}{\sqrt{NT}} + \gamma^2 \bigg\},
    \end{align*}
    where $c > 0$ is some universal numerical constant, and $\sfN_{\infty}(\calH, \gamma)$ is the covering number of $\calH$ at resolution $\gamma$ under the metric $\rho(h_1, h_2) = \sup_{x \in \sfX} \norm{h_1(x) - h_2(x)}$.
\end{lemma}
In particular, \Cref{lem: high probability chaining} suggests that the martingale complexity can be bounded solely as a function of the class $\calH$, and not the statistics of the data. Roughly speaking, we can choose $\gamma$ to be whatever is required such that the log-covering number term is dominant, as $\gamma$ manifests only logarithmically there. To determine what $\calH$ to cover, we use the following localization result from \citet[Theorem 5.1]{ziemann2022learning}.
\begin{restatable}{proposition}{iidLocalization}\label{prop: localization}
    Let \Cref{ass: hypercontractivity} hold with $C_{4\to2}^{1:T}$. Defining $\overline \calH_\star$ as the star-hull\footnote{$\calH_\star \triangleq \mathrm{StarHull}(\calH) = \scurly{\alpha h,\; h \in \calH, \alpha \in [0,1]}$.} of $\overline\calH$, $B(\tau) \triangleq \scurly{h \in \overline \calH_\star \big| \sfE^{1:T} \snorm{h}_2^2 \leq \tau^2}$, and $\partial B(\tau)$ the boundary of $B(\tau)$. Then, there exists a $\tau/\sqrt{8}$-net $\overline \calH_{\star}(\tau)$ in the $\norm{\cdot}_\infty$ of $\partial B(\tau)$ such that
    \begin{align}\label{eq: localization}
        \begin{split}
            &\sfP \brac{\inf_{\overline h \in \overline\calH_\star \setminus B(\tau)} \curly{\widehat \sfE^{1:T}_N \snorm{\overline h}_2^2 - \frac{1}{8}\sfE^{1:T}\snorm{\overline h}_2^2 } \leq 0} \leq \abs{\overline \calH_\star(\tau)} \exp\paren{\frac{-NT}{8C_{4\to2}^{1:T}}}.
        \end{split}
    \end{align}
\end{restatable}
\begin{proof}
    We set up to apply \Cref{prop: lower isometry for given fxn}. By \Cref{ass: hypercontractivity}, we have for all $\overline h \in \overline \calH$,
    \begin{align*}
        \sfE^{1:T} \snorm{\overline h(X)}_2^4 \leq C^{1:T}_{4\to2}\paren{\sfE^{1:T} \snorm{\overline h(X)}_2^2}^2.
    \end{align*}
    It is immediately verifiable that the above hypercontractivity assumption on $\overline \calH$ transfers to the star-hull $\overline \calH_\star$. Therefore, setting $\psi(x) \triangleq \snorm{\overline h(x)}_2^2$ and union bounding over applications of \Cref{prop: lower isometry for given fxn} to each $\overline h \in \calH_\star(\tau)$ yields
    \begin{align*}
        &\sfP \brac{ \exists \overline h \in \overline \calH_\star(\tau) : \;\widehat \sfE^{1:T}_N \snorm{\overline h}_2^2  \leq \frac{1}{2}\sfE^{1:T}\snorm{\overline h}_2^2} \leq \abs{\overline \calH_\star(\tau)} \exp\paren{\frac{-NT}{8C_{4\to2}^{1:T}}}.
    \end{align*}  
    Having established a lower uniform law over the covering, we require a way to transfer the statement to the whole class $\overline \calH_\star$ (outside the localization radius $\tau$). The definition of star-hull allows re-scaling of $\overline \calH$ by $\alpha \in [0,1]$, and thus for any $\overline h \in \overline\calH_\star \setminus B(\tau)$, we may rescale by
    \begin{align*}
        \overline h \mapsto \alpha \overline h, \quad \alpha \triangleq \frac{\tau}{\sqrt{\sfE^{1:T} \snorm{h}_2^2}} < 1,
    \end{align*}
    such that $\alpha \overline h \in \partial B(\tau)$. We note that by the parallelogram law and $\snorm{\cdot}_\infty$-covering definition of $\overline\calH_\star(\tau)$, for any $\overline h \in \partial B(\tau)$, there exists $\overline h_i \in \overline\calH_\star(\tau)$ such that
    \begin{align*}
        &\widehat \sfE^{1:T}_N  \snorm{\overline h}_2^2 + \widehat \sfE^{1:T}_N  \snorm{\overline h - \overline h_i}_2^2 = \frac{1}{2}\widehat \sfE^{1:T}_N  \snorm{\overline h_i}_2^2 + \frac{1}{2}\widehat \sfE^{1:T}_N  \snorm{2\overline h - \overline h_i}_2^2\\
        \implies & \widehat \sfE^{1:T}_N  \snorm{\overline h}_2^2 + \frac{\tau^2}{8} \geq \frac{1}{2}\widehat\sfE_N^{1:T} \snorm{\overline h_i}_2^2.
    \end{align*}
    Therefore, harkening back to the lower uniform law, we have that with probability at least $1 - \abs{\overline \calH_\star(\tau)} \exp\paren{\frac{-NT}{8C_{4\to2}^{1:T}}}$, for any $\overline h \in \partial B(\tau)$, there exists $\overline h_i \in \overline \calH_\star(\tau)$
    \begin{align*}
        \widehat \sfE^{1:T}_N  \snorm{\overline h}_2^2 &\geq \frac{1}{2}\widehat\sfE_N^{1:T} \snorm{\overline h_i}_2^2 - \frac{\tau^2}{8} \tag{parallelogram law} \\
        &\geq \frac{1}{4}\sfE^{1:T} \snorm{\overline h_i}_2^2 - \frac{\tau^2}{8} \tag{lower uniform law} \\
        &= \frac{\tau^2}{4} - \frac{\tau^2}{8} = \frac{\tau^2}{8}. \tag{definition of $\partial B(r)$}.
    \end{align*}
    This implies
    \begin{align*}
        &\sfP \brac{\inf_{\overline h \in \partial B(r)} \curly{\widehat \sfE^{1:T}_N \snorm{\overline h}_2^2 - \frac{1}{8}\sfE^{1:T} \snorm{\overline h}_2^2} \leq 0} \leq \abs{\overline \calH_\star(r)} \exp\paren{\frac{-NT}{8C_{4\to2}^{1:T}}},
    \end{align*}
    where we used the definition of the boundary, $\sfE^{1:T} \snorm{\overline h}_2^2 = \tau^2$. To go from the boundary $\partial B(\tau)$ to $\overline \calH_\star \setminus B(\tau)$, we note that inequalities are unaffected by (non-negative) rescaling, which yields the final result.
\end{proof}
Qualitatively, \Cref{prop: localization} implies that given a localization radius $\tau$, elements of $\calH_\star$ outside the radius satisfy a lower uniform law with high probability, such that the expected estimation error can be bounded by the empirical counterpart, which in turn is bounded by the martingale offset complexity. Meanwhile, elements within the localization radius by definition have estimation error bounded by $\tau^2$.
We note that the star-hull subsumes the original class, and thus for a given $\tau$, we have for any $\overline h \in \overline\calH$, with probability at least $1 - \abs{\overline \calH_\star(\tau)} \exp\paren{-NT/8C_{4\to2}^{1:T}}$:
\begin{align}\label{eq: balancing M complexity and radius r}
    \sfE^{1:T}\snorm{\overline h}_2^2 &\leq \max\scurly{8\sfM_{NT}(\overline\calH_\star(r)), \tau^2} \\
    &\leq \max\scurly{8\sfM_{NT}(\overline\calH_\star), \tau^2}.
\end{align}
Thus, we should choose $\tau$ such that the two terms meet at the desired rate, order-wise. The union bound over $\overline\calH_\star(\tau)$ in the failure probability turns into a burn-in condition on $NT$. As the last step before the final bound, we derive the following covering bound:
\begin{restatable}{lemma}{logcovering}\label{lem: log covering bound star hull}
    Under \Cref{ass: function class boundedness}, and recalling $\overline\calH_\star$ is the star-hull of $\overline\calH$, we have
    \begin{align*}
        \log \sfN_\infty(\overline\calH_\star, \varepsilon) &\leq T\dy r \log\mem\paren{1 + \frac{4B_\calF B_\calG}{\varepsilon}} + \log\mem\paren{1+\frac{2B_\calF B_\calG}{\varepsilon}}  + \log\sfN_\infty\mem\paren{\calG, \frac{\varepsilon}{4 B_\calF}}.
    \end{align*}
\end{restatable}

\begin{proof}
    Firstly, noting that $\overline\calH_\star$ is trivially $2B_\calF B_\calG$-bounded by \Cref{ass: function class boundedness}, we invoke \citet[Lemma 4.5]{mendelson2002improving} to show the covering number of star-hull of $\overline\calH$ incurs only a logarithmic additive factor to the log-covering number.
    \begin{align*}
        \log \sfN_\infty(\overline\calH_\star, \varepsilon) \leq \log \sfN_\infty(\overline\calH, \varepsilon/2) + \log\mem\paren{1+\frac{2B_\calF B_\calG}{\varepsilon}}.
    \end{align*}
    It remains to demonstrate how a covering of $\calF^{\otimes T}$ and $\calG$ witnesses an $\varepsilon$-covering of $\overline\calH$. Given $h_1,h_2 \in \overline\calH$, define the $\norm{\cdot}_\infty$ norm:
    \begin{align*}
        \norm{h_1 - h_2}_\infty &\triangleq \max_{t \in [T]} \sup_{x \in \sfX} \norm{F^{(t)}_1 g_1(x) - F^{(t)}_2 g_2(x)}_2 \\
        &\leq \max_{t \in [T]} \sup_{x \in \sfX} \norm{(F^{(t)}_1 - F^{(t)}_2) g_1(x)}_2 + \norm{F^{(t)}_2 (g_1-g_2)}_\infty \tag{add and subtract, triangle ineq.} \\
        &\leq \max_{t \in [T]} B_\calG \norm{F^{(t)}_1 - F^{(t)}_2}_2  +  B_\calF \norm{g_1 - g_2}_\infty. \tag{Cauchy-Schwarz, boundedness} \\
    \end{align*}
Therefore, to witness a $\varepsilon$-covering of $\overline\calH$, it suffices to cover $\calF$ at resolution $\frac{\varepsilon}{2B_\calG}$ in the Frobenius norm for \textit{each} $t \in [T]$, and $\calG$ at resolution $\frac{\varepsilon}{2B_\calF}$ in the sup-norm $\norm{\cdot}_\infty$. We recall by \Cref{ass: function class boundedness} that $\calF^{\otimes T}$ is identified by the product of $T$ $\dy \times r$-dimensional Frobenius norm-ball of radius $B_\calF$, and thus by standard volumetric arguments (e.g.\ \citet[Example 5.8]{wainwright2019high}), we combine bounds and recover
\begin{align*}
    \log \sfN_\infty(\overline\calH_\star, \varepsilon) &\leq T\dy r \log\mem\paren{1 + \frac{4B_\calF B_\calG}{\varepsilon}} + \log\mem\paren{1+\frac{2B_\calF B_\calG}{\varepsilon}}  + \log\sfN_\infty\mem\paren{\calG, \frac{\varepsilon}{4 B_\calF}}.
\end{align*}
    
\end{proof}

With a covering-number bound in hand, we may now instantiate \Cref{lem: high probability chaining}.
\begin{lemma}[Martingale Complexity Bound]\label{lem: multi-task martingale complexity}
    Let \Cref{ass: conditional-subG noise} hold, and fix $\delta \in (0,1/3]$. Then, the following upper bound holds on the martingale complexity of $\overline \calH_\star$:
    \begin{align*}
        \sfM_{NT}(\overline\calH_\star) &\lesssim \sigma_\sfW^2 \Bigg(\frac{\dy r}{N}\log\paren{e + \frac{B_\calF B_\calG NT}{\sigma_\sfW} } + \frac{d_\theta}{NT}\log\paren{e+\frac{ B_\calF B_\theta L_\theta NT}{\sigma_\sfW}} + \frac{\log(1/\delta)}{NT}  \Bigg),
    \end{align*} 
    assuming $\calG$ is a $(B_\theta, L_\theta)$-Lipschitz parametric function class (\Cref{def: Lipschitz loss class}). When $\calG$ is finite, we instead have the following bound:
    \begin{align*}
        \sfM_{NT}(\overline\calH_\star) &\lesssim \sigma_\sfW^2 \Bigg(\frac{\dy r}{N}\log\paren{e + \frac{B_\calF B_\calG NT}{\sigma_\sfW} } + \frac{\log\abs{G} + \log(1/\delta)}{NT}  \Bigg),
    \end{align*} 
\end{lemma}

\begin{proof}
    Instantiating the high-probability martingale complexity bound from \Cref{lem: high probability chaining} for $\overline\calH_\star$, and further inverting the tail-bound parameters $u, v, w$ for failure probability $\delta \in (0, 1/3]$, we have with probability at least $1 - \delta$:
    \begin{align*}
        \sfM_{NT}(\overline\calH_\star) \leq c \inf_{\gamma > 0, \omega \in [0, \gamma]}\cdot &\bigg\{\omega \sigma_{\sfW}\sqrt{\dy} \log(1/\delta) + \frac{\sigma_{\sfW}}{\sqrt{NT}} \int_{\omega/2}^\gamma \sqrt{\log\sfN_\infty(\overline\calH_\star, \varepsilon)} \;d\varepsilon + \frac{\sigma_{\sfW}^2}{NT}\log(1/\delta) \\
        &\quad +\frac{\sigma_{\sfW}^2 \log\sfN_\infty(\overline\calH_\star, \gamma)}{NT} + \frac{\gamma \sigma_{\sfW}}{\sqrt{NT}}\sqrt{\log(1/\delta)} + \gamma^2 \bigg\}.
    \end{align*}
    To heuristically decide the magnitude of $\gamma, \omega$, we plug in the covering number bound from \Cref{lem: log covering bound star hull} in to yield:
    \begin{align*}
        \frac{\sigma_\sfW^2 \log\sfN_\infty(\overline\calH_\star, \gamma)}{NT} &\leq \sigma_{\sfW}^2 \brac{\frac{\dy r}{N} \log\mem\paren{1 + \frac{4B_\calF B_\calG}{\gamma}} + \frac{\log\mem\paren{1+\frac{2B_\calF B_\calG}{\gamma}}}{NT}  + \frac{\log\sfN_\infty\mem\paren{\calG, \frac{\gamma}{4 B_\calF}}}{NT}} \\
        &\lesssim \sigma_{\sfW}^2 \brac{\frac{\dy r}{N} \log\mem\paren{1 + \frac{4B_\calF B_\calG}{\gamma}} + \frac{\log\sfN_\infty\mem\paren{\calG, \frac{\gamma}{4 B_\calF}}}{NT}}.
    \end{align*}
    In the cases of finite and Lipschitz-parametric classes (\Cref{def: Lipschitz loss class}), recall that
    \begin{align*}
        \text{Finite: }& \log\sfN_\infty\mem\paren{\calG, \varepsilon} \leq \log \abs{\calG}  \\
        \text{Parametric: }& \log\sfN_\infty\mem\paren{\calG, \varepsilon} \leq d_\theta \log\paren{1 + \frac{2B_\theta L_\theta}{\varepsilon}}.
    \end{align*}
    Notably, the dependence on $\gamma$ is logarithmic for the above cases, and thus we choose $\gamma$ rather flexibly. Some more care is required for general nonparametric clases \citep{ziemann2022learning}. We use the parametric covering number for pedagogy, as the finite class bound is independent of the covering resolution. We will find the following integral bound useful.

    \begin{lemma}\label{lem: log integral}
        Let $C > 0$ be a given constant. The following bound holds on the integral:
        \begin{align*}
            \int_0^1 \sqrt{\log\paren{1 + \frac{C}{x}}} \;dx &\leq \sqrt{\int_0^1 \log\paren{1 + \frac{C}{x}} dx } \\
            &\leq \sqrt{\log(e(1+x))}.
        \end{align*}
    \end{lemma}
    \begin{proof}
        The first inequality holds by applying Jensen's inequality on $\int_0^1 \sqrt{\log\paren{1 + \frac{C}{x}}} \;dx = \sfE_{X \sim \mathrm{Unif}[0,1]} \brac{\sqrt{\log\paren{1 + \frac{C}{X}}}}$. The second inequality holds by a routine integration:
        \begin{align*}
            \int_0^1 \log\paren{1 + \frac{C}{x}} dx &= \log(1+C) + C\log\paren{1+\frac{1}{C}} \\
            &= \log\paren{(1+C)\paren{1+ \frac{1}{C}}^C} \\
            &\leq \log(e\paren{1+C}),
        \end{align*}
        where the last line comes from the fact that $\paren{1 + \frac{1}{C}}^C$ converges to $e$ monotonically from below.
    \end{proof}
    
    Now, returning to the martingale complexity bound, it suffices to choose $\gamma = \frac{\sigma_\sfW}{NT}$ and $\omega = 0$ such that
    \begin{align*}
        &\int_{\omega/2}^\gamma \sqrt{\log\sfN_\infty(\overline\calH_\star, \varepsilon)} \;d\varepsilon = \int_{0}^{\gamma} \sqrt{\log\sfN_\infty(\overline\calH_\star, \varepsilon)} \;d\varepsilon \\
        \leq\;& \gamma \int_{0}^{1} \sqrt{T \dy r \log\mem\paren{1 + \frac{4B_\calF B_\calG}{\gamma \varepsilon}} + \log\sfN_\infty\mem\paren{\calG, \frac{\gamma \varepsilon}{4 B_\calF}} }  \;d\varepsilon, \qquad \varepsilon \mapsto \frac{\varepsilon}{\gamma}\\
        &\leq \gamma \sqrt{T\dy r \log\paren{e+ \frac{12B_\calF B_\calG}{\gamma}}} + \gamma \sqrt{d_\theta \log\paren{e + \frac{24B_\calF B_\theta L_\theta}{\gamma}}}
    \end{align*}
    where we used the fact $\sqrt{a+b} \leq \sqrt{a} + \sqrt{b}$ and \Cref{lem: log integral}. In any case, this implies that the bound on $\int_{\omega/2}^\gamma \sqrt{\log\sfN_\infty(\overline\calH_\star, \varepsilon)} \;d\varepsilon$ term is order-wise dominated by the bound on $\log\sfN_\infty(\overline\calH_\star, \gamma)$
    \begin{align*}
        &\frac{\sigma_{\sfW}}{\sqrt{NT}} \int_{\omega/2}^\gamma \sqrt{\log\sfN_\infty(\overline\calH_\star, \varepsilon)} \;d\varepsilon +\frac{\sigma_{\sfW}^2 \log\sfN_\infty(\overline\calH_\star, \gamma)}{NT} \\
        \lesssim\; & \frac{\sigma_{\sfW}^2}{NT} \paren{T\dy r \log\paren{e+\frac{12B_\calF B_\calG NT }{\sigma_\sfW}}  + d_\theta \log\paren{e+\frac{24B_\calF B_\theta L_\theta NT}{\sigma_\sfW}}},
    \end{align*}
    and all the other terms in the martingale complexity bound are order-wise upper bounded by the deviation term $\frac{\sigma_\sfW^2}{NT}\log(1/\delta)$. Therefore, we arrive at the martingale complexity bound:
    \begin{align*}
        \sfM_{NT}(\overline\calH_\star) &\lesssim \sigma_\sfW^2 \Bigg(\frac{\dy r}{N}\log\paren{e + \frac{B_\calF B_\calG NT}{\sigma_\sfW} } + \underbrace{\frac{d_\theta}{NT}\log\paren{e+\frac{ B_\calF B_\theta L_\theta NT}{\sigma_\sfW}}}_{\log\abs{G}/NT\text{ for finite }\calG} + \frac{\log(1/\delta)}{NT}  \Bigg).
    \end{align*}    

\end{proof}

Now, it remains to balance the localization radius $\tau$ to yield a bound the task-averaged estimation error.
\iidEstErrorBound*
\begin{proof}
    Toward choosing the localization radius $\tau$, it suffices to choose $\tau^2 \lesssim \sfM_{NT}(\overline\calH_\star)$, yielding with probability at least $1 - \delta - \abs{\overline\calH_\star(\tau/\sqrt{8})}\exp\paren{-NT/8C_{4\to2}^{1:T}}$,
    \begin{align*}
        \frac{1}{T}\sumT \Et\snorm{\Fthat \hat g(X) - \Ftstar g_\star(X)}_2^2 \leq \max\scurly{\sfM_{NT}(\overline\calH), \tau^2} = \sfM_{NT}(\overline\calH),
    \end{align*}
    for which we provided a bound in \Cref{lem: multi-task martingale complexity}. Toward inverting the failure probability, we observe that it suffices here to choose $\tau = \frac{\sigma_\sfW}{NT}$ matching the choice of $\gamma$ in the proof of \Cref{lem: multi-task martingale complexity}, such that we may recycle computations to yield
    \begin{align*}
        \log\sfN_\infty(\overline\calH_\star, \tau/\sqrt{8})
        &\lesssim \dy r \log\mem\paren{1 + \frac{B_\calF B_\calG}{\tau}} + \log\sfN_\infty\mem\paren{\calG, \frac{\tau}{B_\calF}} \\
        &\lesssim \dy r \log\paren{\frac{B_\calF B_\calG NT}{\sigma_\sfW}} + d_\theta \log\paren{\frac{B_\calF B_\theta L_\theta NT}{\sigma_\sfW}}
    \end{align*}
    Therefore, inverting $\abs{\overline\calH_\star(\tau/\sqrt{8})}\exp\paren{-NT/8C_{4\to2}^{1:T}} \leq \delta$ yields the burn-in requirement
    \begin{align*}
        N\gtrsim C^{1:T}_{4 \to 2} \paren{\dy r \log\paren{\frac{B_\calF B_\calG NT}{\sigma_\sfW}} + \underbrace{\frac{d_\theta}{T} \log\paren{\frac{B_\calF B_\theta L_\theta NT}{\sigma_\sfW}}}_{\log|\calG|/T \text{ for finite } \calG} + \frac{\log(1/\delta)}{T}}.
    \end{align*}
\end{proof}
By \Cref{lemma: excess risk two term decomp}, we simply sum up the bounds from \Cref{prop: iid NRLS bound} and \Cref{prop: iid estimation error bound} and apply \Cref{prop: task coverage -> task diversity} to specify $\nu^{-1}$, which yields
\iidMainBound*
The modified burn-in on $N'$ comes from \Cref{lem: noiselevellemma}, where the additive error from misspecification in $\sigma_U^2$ when expanded is proportional to
\begin{align*}
    \frac{r C_Z\sqrt{C_{4\to2}^{(0)}}}{N'} \frac{\nu^{-1}}{T}\sumT \Et \snorm{\Fthat \hat g(X) - \Ftstar g_\star(X)}_2^2.
\end{align*}
Therefore, it suffices to inflate the existing burn-in on $N'$ by an additive $\approx C_Z \sqrt{C_{4\to2}^{(0)}} r$ factor so that the estimation error terms merge.

To extend bounds to the $\phi$-mixing, beyond the legwork done \Cref{appdx: non-realizable least squares}, very little changes for the estimation error bounds, apart from the sole modification in Samson's Theorem:
\begin{restatable}[{\citet[Theorem 2]{samson2000concentration}}, {\citet[Prop.\ 5.1]{ziemann2022learning}}]{proposition}{mixLowerIsometry} \label{prop: mixing lower isometry for given fxn}
    Fix $C > 0$. Assume $\scurly{X}_{i\geq 1} \sim \sfP$ is $\phi$-mixing and admits dependency matrix $\Gamma_{\mathrm{dep}}(\sfP)$. Let $g : \sfX \to \R$ be a non-negative function satisfying
    \begin{align}
        \sfE [g(X)^2] \leq C \sfE [g(X)]^2. \nonumber
    \end{align}
    Then we have:
    \begin{align}
        \sfP\brac{\frac{1}{m} \sum_{i=1}^m g(x_i) \leq \frac{1}{2} \sfE [g(X)]} \leq \exp\paren{\frac{-m}{8C \norm{\Gamma_{\mathrm{dep}}(\sfP)}_2^2 }}. \nonumber
    \end{align}
\end{restatable}
Using the bound following \Cref{def:depmat}, defining $\Phi\triangleq \paren{\sum_{i=1}^\infty \sqrt{\phi_X(i)}}^2$, we can follow the exact same steps above for the iid case to yield:
\mixMainBound*
Note that the burn-in for $N$ now has an additional factor of $\Phi$.

%% file: icml/icml_tex/mixingapp.tex
\section{Properties of Mixing Sequences of Random Variables}\label{appdx: blocking}

In \Cref{sec: main; rep learning with little mixing} we extend our analysis to mixing random variables. This requires some additional machinery. Namely, for a sequence of random variables $Z_{1:n}$ we partition $[n]$ into $2m$ consecutive intervals, denoted $a_j$ for $j \in [2m]$, so that $\sum_{j=1}^{2m} |a_j| = n$. Denote further by $O$ (resp. by $E$) the union of the oddly (resp. evenly) indexed subsets of $[n]$. We further abuse notation by writing $\phi_Z(a_i) = \phi_Z(|a_i|)$ in the sequel. We will typically instantiate the below machinery with all partitions of equal length $k$, but for now describe the general setup.

We split the process $Z_{1:n}$ as:
\begin{equation}\label{eq:blockstructure}
    Z^o_{1:|O|} \triangleq (Z_{a_1},\dots,Z_{a_{2m-1}}), \quad Z^e_{1:|E|} \triangleq (Z_{a_2},\dots,Z_{a_{2m}}).
\end{equation}
Let $\tilde Z^o_{1:|O|}$ and $\tilde Z^e_{1:|E|}$ be blockwise decoupled versions of \eqref{eq:blockstructure}. That is we posit that $\tilde Z^o_{1:|O|} \sim \sfP_{\tilde Z^o_{1:|O|}}$ and $\tilde Z^e_{1:|E|} \sim \sfP_{\tilde Z^e_{1:|E|}}$, where:
\begin{equation}\label{eq:blocktensor}
    \sfP_{\tilde Z^o_{1:|O|}} \triangleq \sfP_{Z_{a_1}} \otimes \sfP_{Z_{a_3}}\otimes  \dots \otimes \sfP_{Z_{a_{2m-1}}} 
    \quad \textnormal{and}\quad
    \sfP_{\tilde Z^e_{1:|E|}} \triangleq \sfP_{Z_{a_2}} \otimes \sfP_{Z_{a_4}} \otimes\dots \otimes \sfP_{Z_{a_{2m}}}.
\end{equation}
The process $\tilde Z_{1:n}$ with the same marginals as $\tilde Z^o_{1:|O|}$ and $\tilde Z^e_{1:|E|}$ is said to be the decoupled version of $Z_{1:n}$. To be clear: $\sfP_{\tilde Z_{1:n}} \triangleq \sfP_{Z_{a_1}} \otimes \sfP_{Z_{a_2}}\otimes  \dots \otimes \sfP_{Z_{a_{2m}}}$, so that $\tilde Z^o_{1:|O|}$ and $\tilde Z^e_{1:|E|}$ are alternatingly embedded in  $\tilde Z_{1:n}$. The following result is key---by skipping every other block, $\tilde Z_{1:n}$ may be used in place of $Z_{1:n}$ for evaluating scalar functions at the cost of an additive mixing-time-related term.

\begin{proposition}[Lemma 2.6 in \cite{yu1994mixing} instantiated to $\phi$-mixing processes]
\label{prop:decoupling}
Fix a $\phi$-mixing process $Z_{1:n}$ and let $\tilde Z_{1:n}$ be its decoupled version.  For any measurable function $f$ of $Z^o_{1:|O|}$ (resp.\ $g$ of $Z^e_{1:|E|}$)
with joint range $[0,1]$ we have that:
\begin{equation}
    \begin{aligned}
        | \sfE(f(Z^o_{1:|O|})) - \sfE (f(\tilde Z^o_{1:|O|})) | &\leq \sum_{i \in E\setminus \{2m\}} \phi_Z(a_i),\\
        | \sfE(g(Z^e_{1:|E|})) - \sfE (g(\tilde Z^e_{1:|E|})) | &\leq \sum_{i \in O\setminus\{1\}} \phi_Z(a_i).
    \end{aligned}
\end{equation}
\end{proposition}

The above proposition is originally stated for $\beta$-mixing random variables in \citet{yu1994mixing}, but these coefficients always dominate the $\phi$-mixing coefficients and so the result remains true in our setting. 

We will also require a second notion of dependency. 

\begin{definition}[{Dependency matrix, \citet[Section 2]{samson2000concentration}}]
\label{def:depmat}
The \emph{dependency matrix}
of a process $Z_{1:n}$ with distribution $\mathsf{P}_{Z}$ is the (upper-triangular) matrix $\Gamma_{\mathsf{dep}}(\mathsf{P}_{Z})=\{\Gamma_{ij}\}_{i,j=0}^{T-1} \in \mathbb{R}^{n\times n}$ defined as follows.
Let $\mathcal{Z}_{1:i+1}$ denote the $\sigma$-algebra 
generated by $Z_{1:i+1}$.
For indices $i < j$, let 
\begin{align}
    \Gamma_{ij} = \sqrt{ 2 \sup_{A \in \mathcal{Z}_{1:i+1}}\|\sfP_{Z_{j+1:n}}(\cdot \mid A) - \sfP_{Z_{j+1:n}}  \|_{\mathsf{TV}}}. \label{eq:dep_mat_coeffs}
\end{align}
For the remaining indices $i \geq j$, let
$\Gamma_{ii} =1$ and $\Gamma_{ij}=0$ when $i > j$ (below the diagonal).
\end{definition}

It is straightforward to verify---and we will use---that
\begin{equation}\label{eq:gammaboundedbyphi}
    \| \Gamma_{\mathsf{dep}}(\mathsf{P}_{Z})\| \leq \sum_{i=1}^\infty \sqrt{\phi_Z(i)}.
\end{equation}

%% file: icml/icml_tex/appdx_numerical.tex
\section{Additional Numerical Details}\label{appdx: numerical}

We consider the simulation task of balancing a pole atop a cart from visual observations, as pictured in \Cref{fig: example obs}. This experimental setup is used to demonstrate the benefit of multi-task imitation learning (compared to single task imitation learning) for a visuomotor control task. We first describe the system, and how expert policies are generated. We then provide details about the imitation learning and evaluation process. 

\paragraph{System Description:} The pole is balanced by applying a force to the cart along a track. Denoting the position of the cart by $p$ and the angle of the pole by $\theta$, the system evolves according to the following dynamics:
\begin{align*}
    u &= (M+m)(\ddot p + d_p \dot p) + m\ell ((\ddot \theta + d_\theta \dot \theta) \cos\theta - \dot \theta^2 \sin\theta), \\ 
    0 &= m((\ddot p + d_p \dot p) \cos \theta + \ell(\ddot \theta + d_\theta \dot \theta) - g\sin \theta).
\end{align*}
Here, $M$ is the mass of the cart, $m$ is the mass of the pole, $\ell$ is the length of the pole, $g$ is the acceleration due to gravity, $k_p$ is the damping coefficient for cart on the track, and $k_\theta$ is the damping coefficient for the joint of the pole with the cart. The state of this system at time $t$ is denoted $x_t = \bmat{p_t & \dot  p_t & \theta_t & \dot \theta_t}^\top$. These dynamics are discretized via an euler approximation with
stepsize $dt = 0.02$.  The discrete time dynamics will be written $x_{t+1} = f(x_t, u_t)$. We further suppose that we have a camera setup next to the track, directed towards the track and centered at the zero position of the cart. This camera gives us a partial observation of the state at any time: $o_t = \mathrm{camera}(x_t)$. \Cref{fig: example obs} is one such observation generated by the PyBullet simulator when the system is at the origin. We consider a collection of instances of this system by uniformly randomly sampling $M, m, \ell \in [0.5,3.0]\times [0.05, 0.2] \times [1.0, 2.5]$, and setting $g = 9.8$, $k_p = k_\theta  =0.4$. 

\paragraph{Expert Policy Description: } The expert has access to a (noisy) key-point extractor that maps the image observations from the camera to a vector containing the position of the cart-pole joint along the track, the position of the pole tip along the track, and the height of the pole tip above the track. This provides the two keypoints illustrated in \Cref{fig: ideal keypoints}\footnote{In our experiments, the keypoint extractor is a convolutional neural network trained on a $50000$ cartpole images from instances drawn uniformly at random with states having position $p \in [-3,3]$, $\theta \in [-\pi/3, \pi/3]$, and pole lengths $\ell \in [1,2.5]$.}. We denote this noisy observation as $\mathrm{keypoint}(o_t)$. A single keypoint extractor is used by all experts (across the parameter variations of the system), and is trained from labeled data across a variety of parameter settings. After applying the keypoint extractor to the images, the ideal measurements become a simple function of $p$ and $\theta$: they may be written $\bmat{p_t & p_t + \sin(\theta_t) \ell & \cos(\theta_t) \ell}^\top$. As such, we can construct expert controllers using the dynamics of the system by synthesizing LQG controllers\footnote{We use $Q=R=\Sigma_w=\Sigma_v=I$.} for the system linearized about the upright equilibrium point. In particular, for some particular parameter realization, indexed by $h$, the corresponding expert controller generates the force $u_t^\star$ applied to the cart at time $t$ as 
\begin{align*}
     \xi_{t+1} &= A_K^{(h)} \xi_t + B_K^{(h)} \mathrm{keypoint}\paren{\mathrm{camera}(x_t) }\\ 
    u_t^\star &=  C_K^{(h)} \xi_t + D_K^{(h)} \mathrm{keypoint}\paren{\mathrm{camera}(x_t) },
\end{align*}
where $(A_K^{(h)}, B_K^{(h)}, C_K^{(h)}, D_K^{(h)})$ are constructed from two Riccati equation solutions involving the linearized system, and $\xi_t$ is a four dimensional latent state. We assume that when the input applied is applied to the system, there is an unobserved actuation noise added. Therefore, the input applied to the system at time $t$ by the expert controller will be $u_t = u_t^\star + \eta_t$, where $\eta_t \sim\calN(0,0.5)$. 

\paragraph{Imitation Learning Policy Description: } We consider imitation learning agents that operate a short history of camera observations\footnote{We use a history of 8.}. In particular, the learning agent selects inputs as  
\begin{align*}
    \hat u_t = K_{\theta} \paren{\bmat{\mathrm{camera}(x_t) \\ \vdots \\\mathrm{camera}(x_{t-\mathrm{hist}})}}.
\end{align*}
Here $K_\theta$ is a convolutional neural network with parameters $\theta$. In the single task setting, the parameters are specific to the parameter realization for the task at hand. In the multi-task setting, the network parameters are partitioned into a shared component $\theta_{\mathrm{shared}}$ and a task specific component for the final layer, $\theta_{h}$.

\paragraph{First Stage: } The shared parameters in the multi-task setting are jointly trained on a collection of $H$ source tasks.\footnote{We consider three values of $H$: $H=5, 10, 20$.}. The dataset therefore consists of demonstrations from rollouts of the expert controllers generated for $H$ systems with different parameter realizations. Expert demonstrations are obtained from $10$ independent realizations of the actuation noise sequence for each system. The length of the rollout trajectory is $500$ steps (recalling the discretization timestep of $0.02$.)

The multi-task network is jointly trained on the entire collection of source data to minimize the loss 
\begin{align*}
    \sum_{h=1}^H \sum_{i=1}^{10} \sum_{t=1}^{500} \norm{u_t^{(h)}[i] - K_{\theta_{\mathrm{shared}}, \theta_h} \paren{\bmat{\mathrm{camera}(x_t^{(h)}[i]) \\ \vdots \\\mathrm{camera}(x_{t-\mathrm{hist}}^{(h)}[i])}}}^2
\end{align*}
over the network parameters $\theta_{\mathrm{shared}}, \theta_1, \dots, \theta_H$. The superscript $h$ on the inputs and states denotes the system index that they came from, while the argument in the brackets enumerates the $10$ expert trajectories collected from each system. To obtain an approximate minimizer to the above problem, we employ the adam optimizer using a batch size of 32, weight decay of $1e^{-3}$, and learning rate of $1e^{-3}$ with a decay factor of $0.5$ every 10 epochs for a total of 100 epochs.\footnote{Tasks are mixed together in the each batch.}

\paragraph{Second Stage: }
The second stage consists of 10 target tasks, defined by new parameter realizations for the cartpole system. We compare:
\begin{enumerate}
    \item Training a convolutional neural network for each of these tasks from scratch using the data available for the task (this is single task imitation learning).
    \item Re-using the representation trained for the collection of source tasks along with a head that is fit to the target task. The head is obtained by solving a least squares problem by computing the shared representation for the history of camera observations in the expert demonstrations and solving a regression problem to match the expert inputs. 
\end{enumerate}
For each target task, we again collect expert demonstrations. Here, we consider a variable number of trajectories, $N_{\mathrm{target}}$. Each trajectory is again obtained by rolling out the corresponding expert controller for $500$ steps under new, independent realizations of the actuation noise. These expert trajectories are used to fit a linear head for the corresponding target tasks for the multi-task setting, and to train a behavior cloning agent from scratch for the single task setting. 

\paragraph{Evaluation Results: } Once the target controllers are trained, we evaluate them by rolling them out on the cartpole system with the parameters for which they were designed. These evaluation rollouts occur by rolling out out the single-task learned, multi-task learned, and expert controller under new realizations of the actuation noise. We track the input imitation error over the entire trajectory, which is the MSE of the gap between the inputs applied by the expert, and the inputs a learned controller $\hat K$ would apply when faced with the same observations:
\[{\sum_{t=1}^{500} \norm{u_t^\star - \hat K\paren{\bmat{\mathrm{camera}(x_t^\star) \\ \vdots \\ \mathrm{camera}(x^\star_{t-\mathrm{hist}})}}}^2}.\]
We additionally track the state imitation error between the states $\hat x_t$ from rolling out the learned controller and the states $x_t^\star$ from rolling out the expert controller:
\[ {\sum_{t=1}^{500} \norm{x_t^\star - \hat x_t}^2}.\]
We also track whether the controller lasts 500 steps without allowing the pole to fall past an angle of $\pi/2$ in either direction. We plot the results for representation learning with $5$, $10$, or $20$ source tasks, in addition to single task learning. The evaluation metrics are averaged across $50$ evaluation rollouts for each target controller. In \Cref{fig: cartpole imitation learning comparisons}, the median is plotted, with the 30\%-70\% quantiles are shaded. The median and quantiles are over $10$ random seeds for the target tasks and $5$ random seeds for the parameters of the source task instances. In the low data regime, multi-task learning excels in all metrics, with increasing benefit as more source tasks are available. In the high data regime, the single task controller eventually beats out the multi-task controllers for all metrics.

In \Cref{fig: cartpole test loss rate}, we plot the input imitation error versus the number of source tasks available for pre-training on a $\log-\log$ scale with the number of target trajectories fixed at $100$. Neglecting the component of the error that decays with the number of target trajectories, our theoretical results predict a decay in the error of $\frac{1}{H}$, or a slope of $-1$ on a $\log\log$ plot. In \Cref{fig: cartpole test loss rate}, we observe a slope of approximately $-0.8$. The discrepancy may arise for several reasons. Firstly, the empirical risk minimizer is approximated using SGD. Secondly, the number of target trajectories used for fitting the final layer of the network is not infinite, meaning that we occur some additional error in training the final layer. 

\begin{figure}[t]
    \centering
    \includegraphics[width=0.4\textwidth]{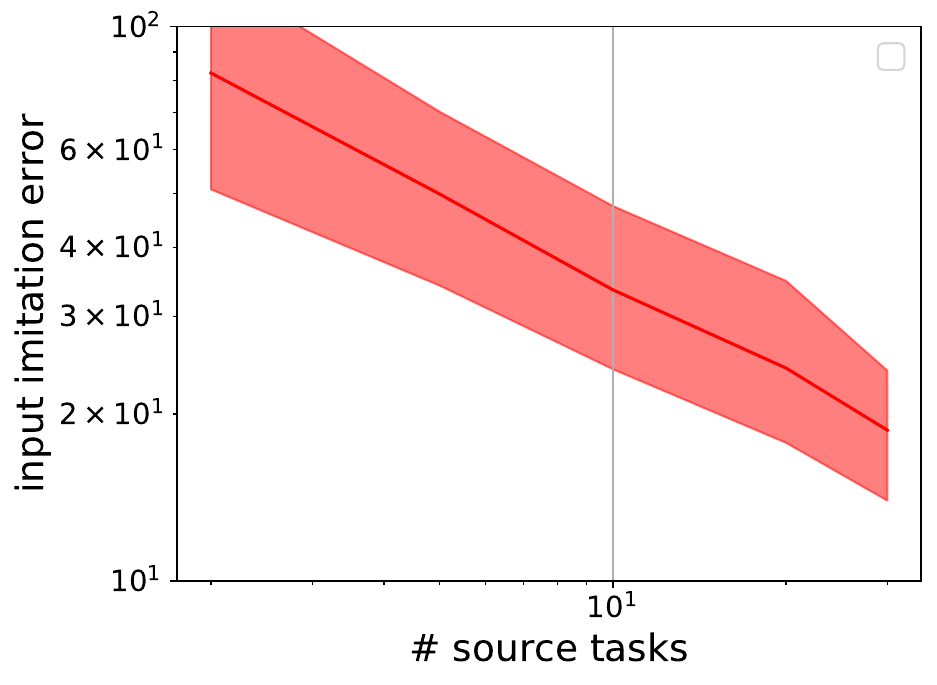}
    \caption{Input imitation error of the policies trained with a shared representation plotted against the number of source tasks used to train the representation on a $\log\log$ scale. The number of target trajectories used for finetuning is fixed at $100$.}
    \label{fig: cartpole test loss rate}
\end{figure}